\documentclass{article}
\usepackage{enumitem}
\usepackage{graphicx}
\usepackage{xcolor}
\usepackage{tikz}
\usepackage{hyperref}
\usepackage{amsmath, amssymb, amsthm}
\usepackage{multirow}
\usepackage{caption}
\usepackage{float} 
\usepackage{booktabs}
\usepackage{algorithm}
\usepackage{algpseudocode}
\usepackage{amsmath}
\usepackage{setspace}
\usepackage{dsfont}

\usepackage[style=numeric,sorting=none,defernums=false]{biblatex}
\bibliography{ref}   

\usepackage{titling}            
\usepackage{bm}
\newtheorem{theorem}{Theorem}[section]
\newtheorem{lemma}{Lemma}[section]
\theoremstyle{definition}
\newtheorem{definition}{Definition}[section]
\newtheorem{corollary}{Corollary}[section]
\newtheorem{remark}{Remark}[section]
\usepackage{titling}   
\usepackage{ragged2e}   

\usepackage{xcolor}
\definecolor{mygreen}{RGB}{0,128,0}

\begin{document}

\title{Chem-NMF: Multi-layer $\alpha$-divergence Non-Negative Matrix Factorization for Cardiorespiratory Disease Clustering, with Improved Convergence Inspired by Chemical Catalysts and Rigorous Asymptotic Analysis}
\author{%
\normalsize
\textbf{Yasaman Torabi}\textsuperscript{1,*} \quad
\textbf{Shahram Shirani}\textsuperscript{1,2} \quad
\textbf{James P.\ Reilly}\textsuperscript{1}
}

\date{%
\footnotesize
\textsuperscript{1}Electrical and Computer Engineering Department, McMaster University, Hamilton, Ontario, Canada\\
\textsuperscript{2}L.R. Wilson/Bell Canada Chair in Data Communications, Hamilton, Ontario, Canada
\textsuperscript{*}Corresponding author: torabiy@mcmaster.ca
}

\maketitle

\begin{abstract}
Non-Negative Matrix Factorization (NMF) is an unsupervised learning method offering low-rank representations across various domains such as audio processing, biomedical signal analysis, and image recognition. The incorporation of $\alpha$-divergence in NMF formulations enhances flexibility in optimization, yet extending these methods to multi-layer architectures presents challenges in ensuring convergence. To address this, we introduce a novel approach inspired by the Boltzmann probability of the energy barriers in chemical reactions to theoretically perform convergence analysis. We introduce a novel method, called Chem-NMF, with a bounding factor which stabilizes convergence. To our knowledge, this is the first study to apply a physical chemistry perspective to rigorously analyze the convergence behaviour of the NMF algorithm. We start from mathematically proven asymptotic convergence results and then show how they apply to real data. Experimental results demonstrate that the proposed algorithm improves clustering accuracy by $5.6\% \;\pm\; 2.7\%$ on biomedical signals and 11.1\% $\pm$ 7.2\% on face images (mean ± std).
\end{abstract}

\noindent \textbf{Index Terms—} Nonnegative matrix factorization, convergence, optimization, physical chemistry, clustering, Boltzmann probability, image recognition, biomedical signal processing, cardiorespiratory disease, heart sound, lung sound, NMF.

\section{Introduction}
Data clustering plays a critical role in computer vision and pattern recognition, as it enables unsupervised organization of large-scale datasets into meaningful groups. Recent clustering methods, such as graph-based learning \cite{10706657}, subspace clustering \cite{MIAO2025346}, and deep learning approaches \cite{3670408}, have achieved remarkable progress, but they often suffer from high computational complexity, sensitivity to noise, or lack of interpretability \cite{Wan2025}. Multi-view clustering methods have been proposed to capture complementary information from different feature spaces, yet they typically require post-processing and may fail to fully exploit intrinsic spatial structures \cite{MOUJAHID2025234}. In this context, Non-negative Matrix Factorization (NMF) is an interpretable representation learning tool for data clustering, which is able to generate low-dimensional features \cite{Barkhoda2025}. NMF is a widely used technique for decomposing high-dimensional data into interpretable low-rank components \cite{Cichocki2009}. It has found applications in various fields such as audio processing, biomedical signal analysis, image recognition, text mining, and blind source separation, making it a valuable tool for extracting meaningful patterns from complex datasets \cite{Torabi2023}. Numerous NMF variants, such as graph-regularized NMF \cite{LI2025111679}, locality-preserving NMF \cite{IMANI202510967457}, and robust distributionally-regularized NMF \cite{GILLIS20224052}, have been developed to improve clustering robustness under noisy or high-dimensional conditions. More recent advances include encoder-decoder NMF with $\beta$-divergence, which integrates autoencoder structures for enhanced cluster separability \cite{Soleymanbaigi2025}, and multi-view tensor decomposition methods that unify representation learning with clustering indicators \cite{Wang2025}.  

Among divergence-based NMF approaches, the $\alpha$-divergence formulation provides a flexible framework that generalizes traditional cost functions and enhances model adaptability in different applications \cite{Yang2012}, \cite{Cichocki2008}. However, extending these formulations to multi-layer architectures introduces additional mathematical complexities, requiring a deeper understanding of their theoretical properties, such as convergence \cite{Cichocki2010}. Several studies have investigated the convergence properties of NMF algorithms, often focusing on different divergence measures and optimization techniques. Gillis and Glineur \cite{Gillis2012} analyzed the convergence of standard NMF with multiplicative updates, proving local convergence under specific conditions but not guaranteeing global optimality. Similarly, Fevotte and Idier \cite{tst} explored Itakura-Saito divergence-based NMF for audio signal decomposition, demonstrating practical convergence. Meanwhile, Zhang et al. \cite{Zhang2019} proposed convergence acceleration techniques for NMF. While these works provide valuable insights, they primarily focus on single-layer architectures, leaving the convergence behaviour of multi-layer NMF largely unexplored. Multi-layer models introduce additional non-linearity and dependencies between layers, making their convergence more challenging to analyze. To address these challenges, our work draws inspiration from physical chemistry, particularly the concepts of energy barriers and Boltzmann probability, to provide a new perspective on the convergence of multi-layer $\alpha$-divergence NMF. Energy barriers represent the obstacles that a system must overcome to transition from one stable state to another \cite{PIETRUCCI201732}. The concept of energy barriers is widely observed in natural phenomena where systems must overcome thresholds to transition between states. For example, in physical systems, this behaviour is analogous to free energy functions in thermodynamics, where different configurations yield varying energy levels that influence system stability \cite{Margrave1955}. Similarly, this approach aligns with the concept of activation energy barriers in chemical reactions, where molecules must overcome specific energy thresholds to proceed, as described by the Arrhenius equation \cite{Laidler1984}. Another common example is chemical reactions, where reactants must surpass an activation energy barrier before transforming into a product \cite{reaction}. Similarly, in machine learning, optimization landscapes often contain local minima, and an algorithm’s ability to escape suboptimal states is crucial for achieving global convergence. For example, in stochastic optimization, simulated annealing mimics the annealing process in metallurgy by starting at a high temperature, allowing for broad exploration, and gradually cooling to settle into an optimal configuration \cite{Kirkpatrick1983}. If the system cools too quickly, it risks becoming trapped in local minima; however, by appropriately tuning the $\alpha$ parameter in $\alpha$-divergence, one can control this cooling process and reduce the likelihood of suboptimal convergence. Furthermore, this optimization strategy parallels quantum tunnelling phenomena, where particles overcome classical barriers. In quantum annealing and quantum Boltzmann machines (QBM), quantum fluctuations facilitate the escape from local minima, a behaviour similar to $\alpha$-divergence-based optimization by adjusting how errors influence learning \cite{Amin2018}. 

Recent studies have applied energy barrier analysis to machine learning convergence, such as in deep neural networks \cite{Chaudhari2019} and energy-based models \cite{Hinton2012}, showing that overcoming energy barriers can accelerate convergence \cite{An2023AdsorbML}. However, to our knowledge, this is the first study to apply an energy-based perspective to analyze the convergence behaviour of multi-layer $\alpha$-divergence NMF. By modelling the optimization process as a system navigating an energy landscape, we introduce an analogy where Boltzmann probability governs the likelihood of escaping local minima, thereby improving the robustness of convergence. Our approach provides a new theoretical foundation for understanding convergence in hierarchical NMF models, overcoming the limitations of previous single-layer NMF studies. By incorporating energy barrier modelling, we design an NMF framework that balances escaping poor local minima (exploration) and converging to meaningful solutions (exploitation). Our proposed Chem-NMF improves optimization compared to plain $\alpha$-NMF, and demonstrate its effectiveness in data clustering.

\section {Methodology}

\subsection{Clinical Background}

In this work, we perform clustering on heart and lung sounds as well as image recognition tasks. To better interpret the extracted features, it is important to consider their medical context. Recent developments in clinical Internet of Things (IoT) systems have enabled precise monitoring of cardiac and respiratory cycles \cite{baraeinejad2023clinical}, \cite{t2024E}. The cardiac cycle consists of systole (contraction) and diastole (relaxation), which is regulated by heart valves to ensure one-way blood flow. Normal sounds include S1 and S2 from valve closure, while extra sounds S3 and S4 arise in early and late diastole and may signal dysfunction, such as coronary artery disease \cite{ash2}. Murmurs are additional noises from turbulent blood flow, often divided into systolic or diastolic types \cite{t2024e}. Meanwhile, the respiratory cycle alternates between inspiration and expiration, driven by the diaphragm and chest muscles. Normal breathing produces smooth sounds, while adventitious lung sounds mark abnormalities such as pneumonia \cite{ash}: crackles are brief popping noises from sudden airway opening, wheezes are continuous high-pitched tones from narrowed passages, rhonchi are low, snoring-like sounds, and pleural rubs are rough noises from inflamed membranes \cite{Torabi2024MEMS}.
\subsection{Theoretical Background}

The standard NMF problem seeks to approximate a data matrix 
\(\mathbf{Y} \in \mathbb{R}_+^{I \times T}\) 
with two matrices 
\(\mathbf{A} \in \mathbb{R}_+^{I \times J}\) 
and \(\mathbf{X} \in \mathbb{R}_+^{J \times T}\) 
such that:

\begin{equation}
    \mathbf{Y} = \mathbf{A}\mathbf{X} + \mathbf{E},
\end{equation}

where \(\mathbf{E} \in \mathbb{R}^{I \times T}\) represents the approximation error, 
\(\mathbf{A}\) denotes the basis matrix (i.e. feature set), and \(\mathbf{X}\) corresponds to the activation map (i.e. importance of each feature). 
All matrices are nonnegative. In NMF, the objective is to minimize the error 
\(\mathbf{E}\) between the original data \(\mathbf{Y}\) and the reconstructed data 
\(\mathbf{A}\mathbf{X}\). Unlike closed-form solutions, an iterative update rule approach 
defines a cost function to measure the difference between these two terms and aims 
to minimize it. The choice of cost function leads to various NMF algorithms; the 
specific variant we focus on utilizes the $\alpha$-divergence, known as $\alpha$-NMF. In multi-layer NMF, the basic mixing matrix \(\mathbf{A}\) is replaced by a set of cascaded matrices. 
It follows an iterative decomposition process. First, we approximate 
\(\mathbf{Y} \approx \mathbf{A}^{(1)}\mathbf{X}^{(1)}\). 
Next, the output \(\mathbf{X}^{(1)}\) serves as the new input, decomposed as 
\(\mathbf{X}^{(1)} \approx \mathbf{A}^{(2)}\mathbf{X}^{(2)}\). 
This process continues, considering only the latest components until a stopping criterion is met. 
The final model is:

\begin{equation}
    \mathbf{Y} \approx \mathbf{A}^{(1)} \mathbf{A}^{(2)} \dots \mathbf{A}^{(L)} \mathbf{X}^{(L)},
\end{equation}

where 

\begin{equation}
    \mathbf{A} = \mathbf{A}^{(1)} \mathbf{A}^{(2)} \dots \mathbf{A}^{(L)}, 
    \quad 
    \mathbf{X} = \mathbf{X}^{(L)}.
\end{equation}

\subsection{Physical Chemistry Background}

In order to motivate the analogy between chemical reactions and the convergence of 
multi-layer $\alpha$-NMF, we review several basic chemical concepts \cite{Atkins2022}. In chemical reactions, the initial molecules that change are called \textit{reactants}, while the final stable molecules formed after completion are referred to as \textit{products}. The driving force behind these transformations is the \textit{Gibbs free energy}, which combines a system’s enthalpy $H$ and entropy $S$ at temperature $T$. At constant temperature and pressure, the direction of a reaction is determined by the change in free energy $\Delta G$. A negative $\Delta G$ indicates a spontaneous reaction, while a positive $\Delta G$ requires external energy:
\begin{equation}
\Delta G = \Delta H - T \Delta S .
\end{equation}

Many reactions proceed in \textit{multi-stage reactions}, each with its own transition state and energy barrier (Fig.~\ref{fig:reaction}a). A free energy diagram shows reactants moving through several intermediates before reaching a stable product state. Each stage resembles an energy basin separated by barriers. The \textit{transition state} itself is a high-energy, unstable configuration that represents the maximum energy barrier between reactants and products. Between two such barriers, a temporary species known as an \textit{intermediate} can form. The energy needed to cross the transition state is called the \textit{activation energy}. The Gibbs free energy difference between the reactants and the TS defines the activation barrier, which controls the reaction rate:
\begin{equation}
\Delta G^{\ddagger} = G_{\text{TS}} - G_{\text{reactants}}.
\end{equation}
Catalysts lower $\Delta G^{\ddagger}$ by stabilizing the transition state (Fig.~\ref{fig:reaction}b). In catalyzed reactions, the pathway is rerouted to reduce the activation barrier (e.g. see the catalytic hydrogenation of alkenes in the Supplementary Material).
The likelihood of crossing these barriers is governed by the \textit{Boltzmann distribution}, which describes the probability of a system occupying a state with energy $E$ and thereby determines how easily the system can overcome energy barriers to reach more stable states. 
\begin{figure}[H]
    \centering
    \includegraphics[width=\textwidth]{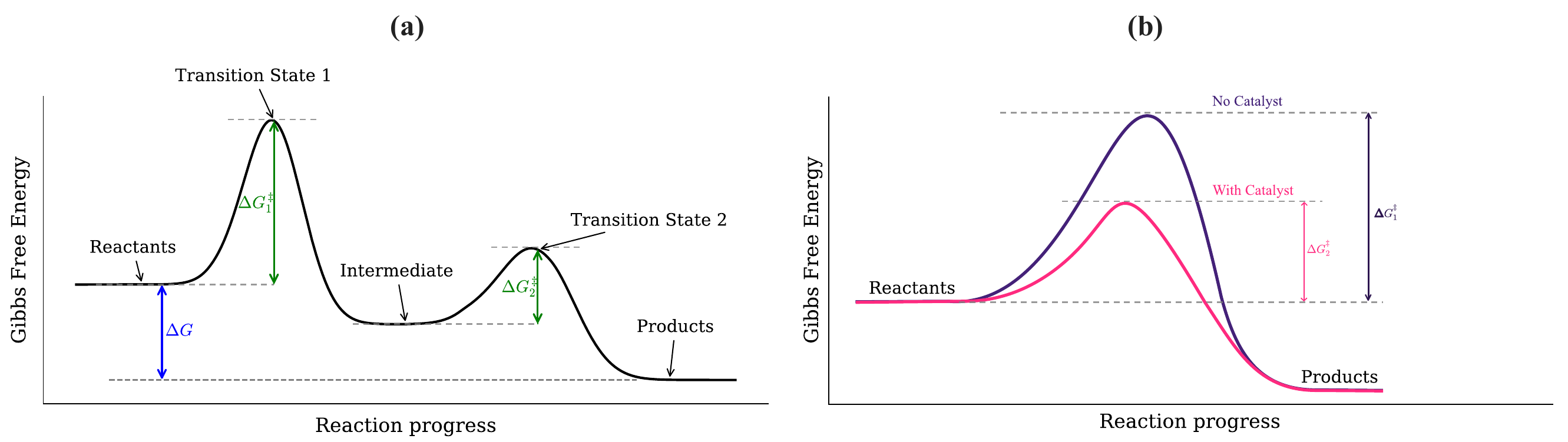}
    \captionsetup{font=footnotesize}
    \caption{Energy profile of a reaction progress: \textbf{(a)} reactants, intermediate, and products. 
    The two transition states (TS1 and TS2) correspond to the energy maxima, with activation 
    free energies $\Delta G^{\ddagger}_{1}$ and $\Delta G^{\ddagger}_{2}$ indicated by vertical 
    arrows. The overall free energy change $\Delta G$ is shown between reactants and products. \textbf{(b)} catalyst effect on lowering the activation barrier.}
    \label{fig:reaction}
\end{figure}
The above chemical phenomena provide a natural analogy to the optimization process of the multi-layer $\alpha$-NMF algorithm (Table~\ref{tab:chem}). In this analogy, chemical reaction pathways and their free-energy landscapes are mapped to the cost-function landscape of multi-layer optimization. Each successive stage represents either a local or global descent step, similar to intermediates in multistage reactions. Thus, just as chemical systems move through intermediates before reaching the most stable state, multi-layer $\alpha$-NMF traverses successive layers to escape shallow minima and converge to better solutions.

\begin{table}[H]
\footnotesize
\centering
\captionsetup{font=footnotesize}
\caption{Analogy between chemical reactions and multi-layer $\alpha$-NMF optimization.}
\label{tab:chem}
\begin{tabular}{p{0.45\linewidth} p{0.45\linewidth}}
\toprule
\textbf{Chemistry Concept} & \textbf{Algorithm Concept} \\
\midrule
Reactants & Input data \\
Transition state & Initial value \\
Intermediate & Local minima in hidden layers \\
Products & Low-rank outputs \\
Gibbs free energy & Optimization cost function \\
Free energy minimum (stable product) & Global minimum of the cost function \\
Boltzmann probability & Escape probability from poor minima \\
Multistage decomposition pathway & Multi-layer factorization trajectory \\
Catalyst lowering barrier & Bounding factor stabilizing convergence \\
\bottomrule
\end{tabular}
\end{table}

The novelty of Chem-NMF lies in the introduction of a bounding factor inspired by catalysts in chemical reactions. Just as catalysts reduce activation barriers and regulate the reaction rate (Fig.~\ref{fig:reaction}b), the bounding factor controls the initialization at the start of each layer and stabilizes the algorithm's convergence. 

\subsection{Proposed Method}
Figure~\ref{fig:method} shows an overview of the procedure. We first factorize the input data into a low-rank basis and an activation map using NMF, and perform clustering with $k$-means on the activation maps. Then, we reconstruct clustered activation maps by multiplying the feature basis matrices. Finally, we evaluate the clustering results using accuracy (ACC) and normalized mutual information (NMI). 

\begin{figure}[H]
    \centering
    \includegraphics[width=\textwidth]{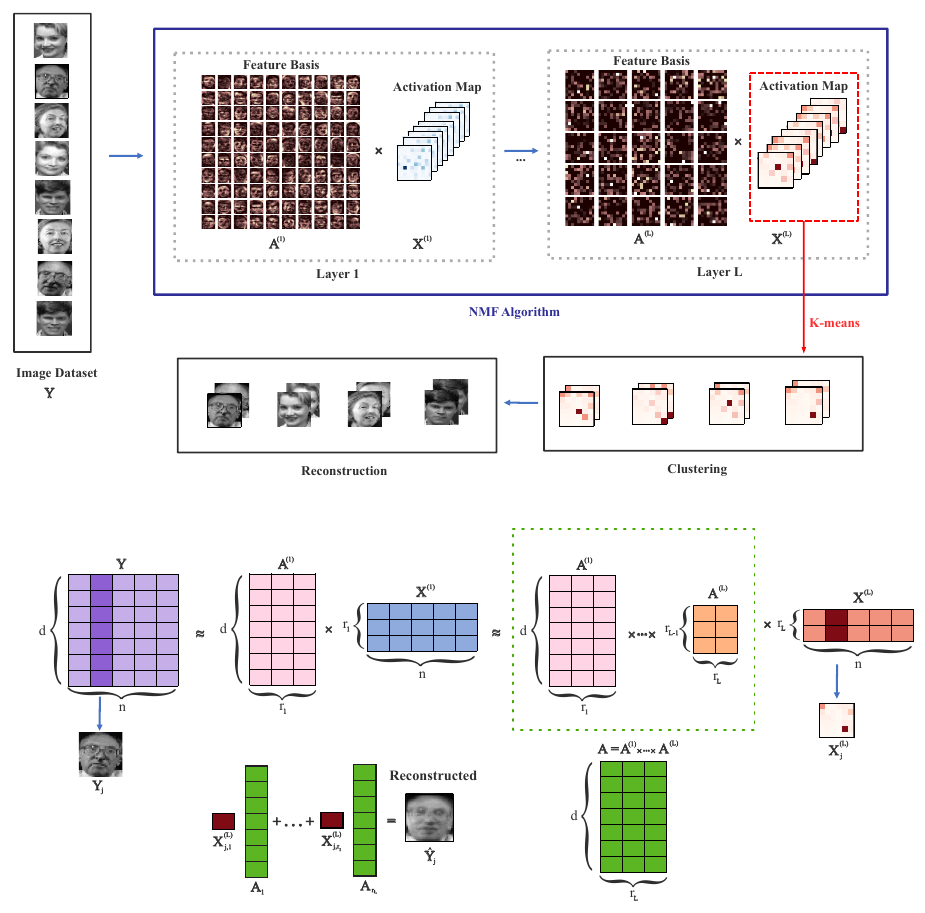}
    \captionsetup{font=footnotesize}
    \caption{Overview of the clustering procedure. 
    The input dataset $\mathbf{Y}$ is factorized into a feature basis $\mathbf{A}$ and activation maps $\mathbf{X}$ across multiple layers using NMF. 
    The activation maps are clustered with $k$-means, and images are reconstructed using feature basis matrices.}
    \label{fig:method}
\end{figure}

Algorithm~\ref{alg:chemnmf} illustrates the proposed algorithm. Chem-NMF is a multi-layer $\alpha$-divergence NMF algorithm that introduces a bounding factor ($BF$) as a novel mechanism to improve convergence. Inspired by the way catalysts reduce activation energy in chemical reactions, the bounding factor is applied during the random initialization step to stabilize the search space in order to reduce the risk of overshooting and getting trapped in local minima.

\begin{algorithm} [H]
\caption{Chem-NMF}
\label{alg:chemnmf}
\begin{spacing}{1}
\begin{algorithmic}[1]
\Require Input data $\mathbf{Y}\in\mathbb{R}_+^{I\times T}$, layer rank $\mathbf{R}=[R_1,..,R_L]$, $\alpha$, $bf$
\Ensure Output activation map $\mathbf{X}^{(L)} \in \mathbb{R}^{R_L \times T}$, basis features $\mathbf{A}_{tot} \in \mathbb{R}^{I \times R_L}$

\State $\mathbf{Y}^{(0)} = \mathbf{Y}$; \quad $\mathbf{A}_{tot} = \mathbf{I}_I$
\For{$\ell=1$ \textbf{to} $L$}
    \State \textbf{Initialization:}
    \If{$\ell=1$}
        \State Random $\mathbf{A}^{(1)} \in \mathbb{R}_+^{I\times R_1}$, $\mathbf{X}^{(1)} \in \mathbb{R}_+^{R_1\times T}$
    \Else
        \State Random $\mathbf{X}^{(\ell)} \in \mathbb{R}_+^{R_\ell\times T}$
        \State Random $\mathbf{A}_{rand}\in \mathbb{R}_+^{R_{\ell-1}\times R_\ell}$
        \State $\mathbf{A}_{base} = \mathrm{mean}(\mathbf{A}^{(\ell-1)}) \cdot \mathds{1}_{R_{\ell-1}\times R_\ell}$
        \State $\mathbf{A}^{(\ell)} = (1-bf)\mathbf{A}_{rand} + bf\mathbf{A}_{base}$
    \EndIf
    \Repeat
       \State $\widehat{\mathbf{Y}}^{(\ell-1)} \gets \mathbf{A}^{(\ell)}\mathbf{X}^{(\ell)}$

    \State $\widehat{\mathbf{Y}}^{(\ell-1)} \gets \mathbf{A}^{(\ell)} \mathbf{X}^{(\ell)}$

    \State $\mathbf{X}^{(\ell)} \gets \mathbf{X}^{(\ell)} \odot
        \left(
          \frac{ (\mathbf{A}^{(\ell)})^\top 
                 \left( \mathbf{Y}^{(\ell-1)} \oslash \widehat{\mathbf{Y}}^{(\ell-1)} \right)^{\alpha} }
               { (\mathbf{A}^{(\ell)})^\top \mathds{1}_I \, \mathds{1}_T^\top }
        \right)^{1/\alpha}$
    
    \State $\mathbf{A}^{(\ell)} \gets \mathbf{A}^{(\ell)} \odot
        \left(
          \frac{ \left( \mathbf{Y}^{(\ell-1)} \oslash \widehat{\mathbf{Y}}^{(\ell-1)} \right)^{\alpha}
                 (\mathbf{X}^{(\ell)})^\top }
               { \mathds{1}_I \, (\mathbf{X}^{(\ell)} \mathds{1}_T)^\top }
        \right)^{1/\alpha}$

        \State Normalize $\mathbf{A}^{(\ell)}$, $\mathbf{X}^{(\ell)}$
    \Until a stopping criterion is met
    \State $\mathbf{A}_{tot} \gets \begin{cases}
        \mathbf{A}^{(1)}, & \ell=1 \\
        \mathbf{A}_{tot}\mathbf{A}^{(\ell)}, & \ell>1
    \end{cases}$
    \State $\mathbf{Y}^{(\ell)} \gets \mathbf{X}^{(\ell)}$
\EndFor
\State \Return $\mathbf{A}_{tot}, \mathbf{X}^{(L)}, \{\mathbf{A}^{(\ell)}\}, \{\mathbf{X}^{(\ell)}\}$
\end{algorithmic}
\end{spacing}
\end{algorithm}

\section{Rigorous Convergence Analysis}
In this section, we mathematically prove that Chem-NMF reduces the probability of converging to local minima. First, we perform convergence analysis for the single-layer case, and then we proceed to the multilayer case.

Let $D_{\alpha}(\mathbf{Y} \parallel \mathbf{A}\mathbf{X})$ denote the objective function based on the $\alpha$-divergence between $\mathbf{Y}$ and $\mathbf{A}\mathbf{X}$ defined as ~(\ref{eq:div}). We show that the algorithm converges subject to its multiplicative update rule \cite{Cichocki2009}.
\begin{equation}
D_{\alpha}(\mathbf{Y} \parallel \mathbf{A}\mathbf{X}) = \frac{1}{\alpha(\alpha - 1)} 
\sum_{it} \left( y_{it}^\alpha [\mathbf{A}\mathbf{X}]_{it}^{1 - \alpha} - \alpha y_{it} + (\alpha - 1) [\mathbf{A}\mathbf{X}]_{it} \right).
\label{eq:div}
\end{equation}

\begin{theorem} 
The NMF algorithm follows the multiplicative update rules: 

\begin{equation}
    x_{jt} \leftarrow x_{jt} 
    \left( \frac{\sum\limits_{i} a_{ij} \left( \frac{y_{it}}{[\mathbf{A}\mathbf{X}]_{it}} \right)^{\alpha}}
    {\sum\limits_{i} a_{ij}} \right)^{\frac{1}{\alpha}}.
\end{equation}

\begin{equation}
    a_{ij} \leftarrow a_{ij} 
    \left( \frac{\sum\limits_{t} x_{jt} \left( \frac{y_{it}}{[\mathbf{A}\mathbf{X}]_{it}} \right)^{\alpha}}
    {\sum\limits_{t} x_{jt}} \right)^{\frac{1}{\alpha}}.
\end{equation}
\end{theorem}
\textit{Proof} Appendix A.

\begin{definition}[Auxiliary Function]
A function $G(\mathbf{X}, \mathbf{X}')$ is an auxiliary function for $F(\mathbf{X})$ if it satisfies the following conditions:
\begin{enumerate}[label=\roman*]
    \item $G(\mathbf{X}, \mathbf{X}) = F(\mathbf{X})$,
    \item $G(\mathbf{X}, \mathbf{X}') \geq F(\mathbf{X})$, for all $\mathbf{X}'$.
\end{enumerate}
\end{definition}

\begin{lemma}
The function
\begin{equation}
G(\mathbf{X}, \mathbf{X}') = \frac{1}{\alpha(\alpha - 1)} 
\sum_{ijt} y_{it} \zeta_{itj} 
\left[ \left( \frac{a_{ij}x_{jt}} {y_{it} \zeta_{itj}}  \right)^{(1-\alpha)} 
+ (\alpha - 1)  \frac{a_{ij} x_{jt}} {y_{it}\zeta_{itj}}  - \alpha \right],
\end{equation}
where 
\begin{equation}
\zeta_{itj} = \frac{a_{ij} x'_{jt}}{\sum_{j=1}^{J} a_{ij} x'_{jt}},
\label{eq:zeta}
\end{equation}
is an auxiliary function for
\begin{equation}
F(\mathbf{X}) = \frac{1}{\alpha(\alpha - 1)} 
\sum_{it} \left( y_{it}^{\alpha} [\mathbf{A}\mathbf{X}]_{it}^{1-\alpha} - \alpha y_{it} + (\alpha - 1) [\mathbf{A}\mathbf{X}]_{it} \right).
\end{equation}
\end{lemma}
\textit{Proof} Appendix B.

\begin{theorem}
$F(\mathbf{X})$ is non-increasing such that:
\begin{equation}  
    F(\mathbf{X}^{(t+1)}) \leq G(\mathbf{X}^{(t+1)}, \mathbf{X}^{(t)}) 
    \leq G(\mathbf{X}^{(t)}, \mathbf{X}^{(t)}) = F(\mathbf{X}^{(t)}).
    \label{eq:Gfunc}
\end{equation}  
\end{theorem}
\textit{Proof} Appendix C.

\noindent The convergence analysis of the update rule for $a_{ij}$ is similar. 

 Although we proved that the algorithm has a non-increasing cost function that guarantees convergence, it may still get trapped in local minima due to the non-convexity of the optimization landscape. 
We show that multi-layer NMF with bounded initialization reduces the probability of converging to local minima. 
However, it requires more iterations, leading to slower convergence. 
This behaviour aligns with the exploration–exploitation trade-off, which means balancing 
between exploration (freely searching for the best solution) and exploitation 
(improving the best-known solution). Exploitation speeds up convergence but may get stuck in local minima, 
while exploration reduces this risk by searching widely but slows convergence, 
requiring more iterations.

\begin{definition}[Energy Barrier]
Let $D_{\alpha}(\mathbf{Y} \parallel \mathbf{A}\mathbf{X})$ be the cost function associated with the NMF algorithm. 
The energy barrier $\xi$ is defined as the difference between the highest cost encountered along an 
optimization path $\gamma$ and the cost at the global minimum:
\begin{equation}
\xi = \max_{\gamma} D_{\alpha}(\mathbf{Y} \parallel \mathbf{A}\mathbf{X}) 
      - D_{\alpha}(\mathbf{Y} \parallel \mathbf{A}^*\mathbf{X}^*),
\end{equation}
where $(\mathbf{A}^*, \mathbf{X}^*)$ is the global minimum solution, 
and $\gamma$ is a transition path in the optimization landscape.
\end{definition}

\begin{definition}[Boltzmann Probability]
The probability of escaping from a local minimum is given by:
\begin{equation}
P = \frac{1}{Z} e^{-\beta \xi},
\end{equation}
where $Z > 0$ is a normalization constant, $\beta > 0$ is an inverse temperature parameter 
that controls stochastic exploration, and $\xi$ is the energy barrier that must be overcome 
to escape local minima.
\end{definition}

\begin{lemma} 
Let $D_l$ represent the $\alpha$-divergence at layer $l$:
\begin{equation}
    D_l = D_{\alpha}(\mathbf{X}^{(l-1)} \parallel \mathbf{A}^{(l)} \mathbf{X}^{(l)}).
\end{equation}
Then, for all $l > 1$, we have:
\begin{equation}
    D_l \leq D_{l-1}.
\end{equation}
\end{lemma}
\textit{Proof} Appendix D.

\begin{theorem} 
Let $P_l$ represent the probability of escaping from a local minimum
at layer $l$. Then, for all sufficiently large $l$ we have: 
\begin{equation}
    P_l \;\ge\; P_{l-1}.
\end{equation}
\end{theorem}

\begin{proof}
Let $M_l$ be the maximum divergence along the optimization path at layer $l$. 
Assume $M_l$ is non-increasing for all sufficiently large $l$. Define:
\begin{equation}
    \mu_l = M_{l-1} - M_l,
\end{equation}
\begin{equation}
    \delta_l = D_{l-1} - D_l,
\end{equation}
\begin{equation}
    \xi_l = M_l - D_{l-1}.
\end{equation}
Then we have:
\begin{equation}
\begin{aligned}
\forall\ l > 1:\quad 
\xi_l - \xi_{l-1}
&= (M_l - D_{l-1}) - (M_{l-1} - D_{l-2}) \\
&= (M_l - M_{l-1}) + (D_{l-2} - D_{l-1}) \\
&= -\,\mu_l - \delta_{l-1}.
\end{aligned}
\end{equation}
By Lemma 4.1, $D_l$ is non-increasing for all $l > 1$, hence:
\begin{equation}
    \forall\ l \ge 3:\quad \delta_{l-1} \ge 0.
\end{equation}
Since $M_l$ is non-increasing for all sufficiently large $l$:
\begin{equation}
    \exists\,L_M \in \mathbb{N}\ \text{such that}\ \forall\,l \ge L_M:\ \mu_l \ge 0.
\end{equation}
Set $L^* := \max\{L_M, 3\}$. Then we have:
\begin{equation}
\begin{aligned}
\forall\,l \ge L^*:\quad & \mu_l \ge 0,\ \delta_{l-1} \ge 0 
\ \Longrightarrow\ -\,\mu_l - \delta_{l-1} \ge 0 \\
& \Longrightarrow\ \xi_l - \xi_{l-1} \le 0 \ \Longrightarrow\ \xi_l \le \xi_{l-1} \\
& \Longrightarrow\ \tfrac{1}{Z} e^{-\beta \xi_l} \ge \tfrac{1}{Z} e^{-\beta \xi_{l-1}} 
\ \Longrightarrow\ P_l \ge P_{l-1}.
\end{aligned}
\end{equation}
\qedhere \text{ QED.}
\end{proof}

\begin{corollary}
Thus, the probability of escaping a local minimum is higher in the multi-layer model, which implies that multi-layer NMF reduces the probability of being trapped in a local minimum. The energy barrier of the final layer of a multi-layer algorithm is smaller than the energy barrier of a single layer, and the probability of escaping from local minima is higher. As it is easier to overcome the barrier and freely explore the energy landscape, the probability of being stuck in a local minimum is lower for multi-layer NMF than for single-layer NMF.
\end{corollary}

\begin{corollary}
Although the final energy barrier decreases across layers, the accumulation of non-negative energy barriers in a multi-layer algorithm results in a higher total energy barrier compared to a single-layer model. Consequently, convergence slows down, as more iterations are required to overcome the cumulative energy barriers and explore the energy landscape in search of the global minimum. This aligns with the exploration–exploitation trade-off, where the improved exploration in multi-layer NMF enhances the ability to escape local minima but comes at the cost of slower exploitation, requiring more steps to refine the optimal solution.
\begin{equation}
\xi_{ML} = \sum_{l=1}^{L} \xi_{l} 
= \xi_{1} + \sum_{l=2}^{L} \xi_{l} 
= \xi_{S} + \sum_{l=2}^{L} \xi_{l} 
> \xi_{S},
\end{equation}
where $\xi_{S}$ and $\xi_{ML}$ are the total energy barriers of single-layer and multi-layer NMF, respectively.
\end{corollary}

\begin{lemma} 
The escape probability $P_l$ converges to a finite value:
\begin{equation}
\lim_{l\to\infty} P_l \;=\; P_\infty.
\end{equation}
\end{lemma}
\textit{proof} Appendix E.

\begin{theorem} 
Across multiple attempts, the multi-layer NMF algorithm has a smaller probability of remaining trapped in a local minimum compared to the single-layer NMF algorithm.
\end{theorem}

\begin{proof}
Let $L_e$ denote the number of attempts at which the algorithm escapes a local minimum. 
For each layer $l \in \mathbb{N}$, define the survival event as:
\begin{equation}
S_l = \{L_e > l\}, 
\end{equation}
which means the process has not yet escaped any local minimum by layer $l$. \\

\begin{remark}
We interpret each attempt as one run of the algorithm at a given layer. 
Thus, the $l$-th attempt corresponds to applying the algorithm at layer $l$. 
In the multi-layer setting, the algorithm proceeds through successive layers, 
while in the single-layer setting, all attempts are confined to the same layer. 
The total number of attempts is denoted by $n$, meaning the algorithm has been 
applied up to layer $n$.
\end{remark}

\begin{lemma}  
For all $n \geq 1$, the survival probability $\mathbb{P}(S_n)$ satisfies:
\begin{equation}
   \mathbb{P}(S_n) = \prod_{l=1}^{n} (1-P_l). 
\end{equation}
\end{lemma}
\textit{proof} Appendix F. \\ \\

By Lemma 4.2, $\lim\limits_{l\to\infty} P_l \;=\; P_\infty.$ 
The formal definition of the limit implies:
\begin{equation}
\forall \, \varepsilon > 0, \; \exists \, l_\varepsilon \in \mathbb{N} 
\;\; \text{such that} \;\; \forall\, l \ge l_\varepsilon 
\;\Longrightarrow\; P_\infty - P_l \le \varepsilon.
\end{equation}
Recall Lemma~4.3 and split the product at $l_\varepsilon$. Then, for any $n \ge l_\varepsilon$ we have:
\begin{align}
\mathbb{P}(S_n)
&= \prod_{l=1}^{n}(1-P_l) \notag \\
&= \underbrace{\Big(\prod_{l<l_\varepsilon}(1-P_l)\Big)}_{C_\varepsilon}\,
   \prod_{l=l_\varepsilon}^{n}(1-P_l) \notag \\
&\le C_\varepsilon \prod_{l=l_\varepsilon}^{n}\bigl(1-(P_\infty-\varepsilon)\bigr)
\qquad\text{(since $P_l \ge P_\infty - \varepsilon$)} \notag \\
&= C_\varepsilon\,\bigl(1-(P_\infty-\varepsilon)\bigr)^{\,n-l_\varepsilon+1}.
\label{eq:bound}
\end{align}
Hence,
\begin{equation}
\mathbb{P}(S_n) \;\le\; C_\varepsilon \,\bigl(1-(P_\infty-\varepsilon)\bigr)^{\,n-l_\varepsilon+1}.
\end{equation}

Let $\widehat{S}_n$ denote the survival event in the single-layer case. Then:
\begin{equation}
\mathbb{P}(\widehat{S}_n) = (1-P_1)^n.
\end{equation}
Since $P_\infty > P_1$, for any $\varepsilon \in (0,\, P_\infty - P_1)$ we have:
\begin{equation}
    1-(P_\infty-\varepsilon) < 1-P_1.
\end{equation}
Thus, for sufficiently large $n$ we obtain:
\begin{equation}
\mathbb{P}(S_n) \;\le\; 
C_\varepsilon \,\bigl(1-(P_\infty-\varepsilon)\bigr)^{\,n-l_\varepsilon+1}
\;<\; (1-P_1)^n \;=\; \mathbb{P}(\widehat{S}_n).
\end{equation}

This implies that across multiple attempts, the multi-layer NMF algorithm has a smaller probability of remaining trapped in a local minimum compared to the single-layer NMF algorithm. 
\qedhere \text{ QED.}
\end{proof}

\section{Experimental Results}
\subsection{Datasets}

We use two image recognition and two bioacoustic datasets for data clustering in different applications: 
face recognition, handwritten digit recognition, cardiac disease detection, and respiratory disease detection. 

For image recognition, we employ the ORL face \cite{orl} and the MNIST handwritten digit \cite{mnist} datasets. Fig.~\ref{fig:dataset} shows sample images from the ORL and MNIST datasets under clean and noisy conditions.
The ORL dataset consists of 400 grayscale facial images from 40 subjects, with 10 images per subject captured under varying conditions, all resized to $32 \times 32$ pixels. 
The MNIST dataset contains 70{,}000 grayscale images of handwritten digits (`0`–`9`); 
For our experiments, we construct a balanced subset of 400 samples (40 per digit), each normalized to $28 \times 28$ pixels. 

\begin{figure}[H]
    \centering
    \includegraphics[width=0.7\textwidth]{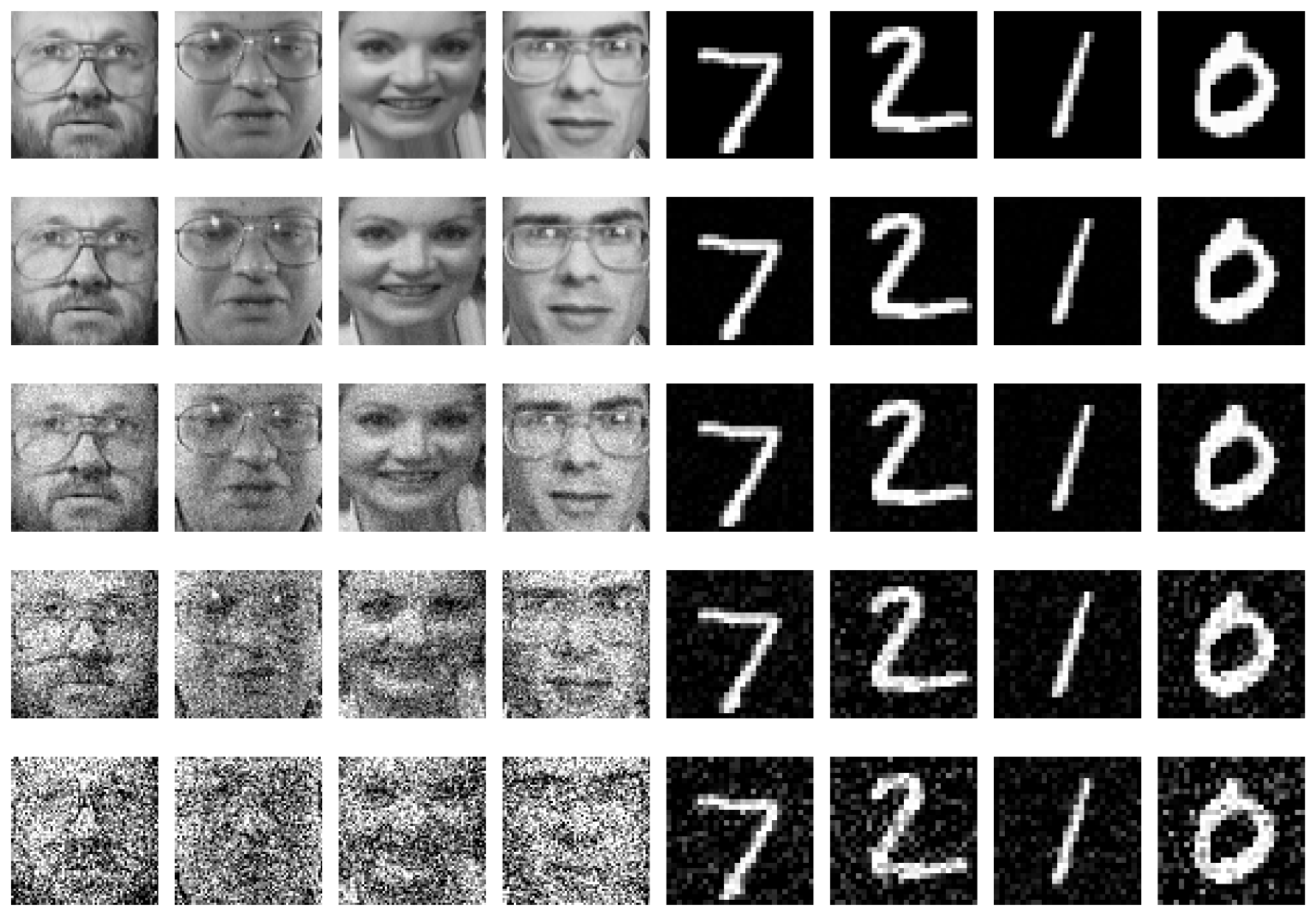}
    \captionsetup{font=footnotesize}
    \caption{Example images from the ORL face and MNIST digit datasets under clean and noisy conditions. 
    From top to bottom: clean images followed by Gaussian noise at 30, 20, 10, and 5~dB SNR levels.}
    \label{fig:dataset}
\end{figure}

In addition to image recognition applications, we cluster heart and lung abnormal sounds. We use the HLS-CMDS dataset \cite{10981596}, which is divided into heart and lung subsets, and it covers normal and abnormal sounds (e.g., atrial fibrillation, wheezing, etc). We recorded the sounds using the 3M™ Littmann CORE Digital Stethoscope from a CAE Juno™ manikin in a quiet clinical simulation lab, placing the stethoscope on standard auscultation landmarks (apex, sternal borders for the heart; upper, middle, and lower anterior chest zones for the lungs). The manikin sounds are pre-recorded from real patients and therefore already include natural noise characteristics such as clothing friction and motion artifacts. During our recordings, we kept the stethoscope steady to minimize handling noise. Recordings were conducted in a quiet environment to further reduce ambient noise. The lung subset consists of 50 recordings, divided into 6 classes (Fig~\ref{fig:hls}a). The heart subset contains 50 recordings of cardiac sounds, categorized into 10 classes (Fig.~\ref{fig:hls}b). Each audio clip is 15\,s long, sampled at 22,050 Hz, and provided in \texttt{.wav} format with metadata. All heart and lung recordings are transformed into time-frequency spectrograms using the short-time Fourier transform (STFT) with a sampling rate of 4\,kHz, a 512-point FFT window, and a hop length of 128, resulting in spectrograms of size $257 \times 470$. The dataset is publicly available, with details of the recording device, sampling rate, sensor placement, environment, and annotated sound categories provided in \cite{10981596}. 

\begin{figure}[H]
    \centering
    \includegraphics[width=0.9\textwidth]{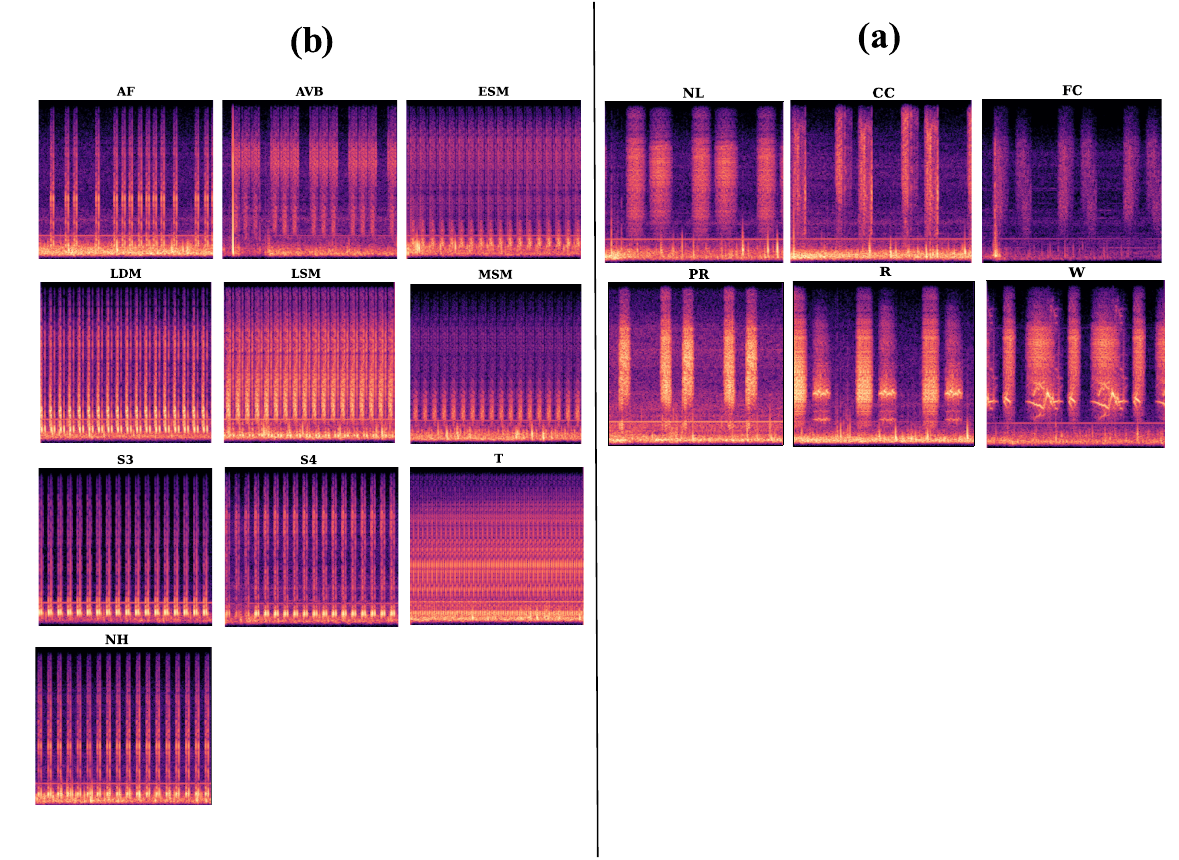}
    \captionsetup{font=footnotesize}
    \caption{Time--frequency spectrograms from the HLS-CMDS dataset: 
    \textbf{(a)} Lung sounds: CC: coarse crackles, FC: fine crackles, N: normal breathing, PR: pleural rub, R: rhonchi, W: wheeze; 
    \textbf{(b)} Heart sounds: AF: atrial fibrillation, AVB: atrioventricular block, ESM: ejection systolic murmur, 
    LDM: late diastolic murmur, LSM: late systolic murmur, MSM: mid-systolic murmur, NH: normal heart sound, 
    S3: third heart sound, S4: fourth heart sound, T: tricuspid insufficiency.}
    \label{fig:hls}
\end{figure}

\subsection{Parameter Sensitivity Analysis}
Fig.~\ref{fig:alpha} shows how $\alpha$ changes the convergence paths of the $\alpha$-divergence surface $D_\alpha(X_1,X_2)$. For $\alpha=-1$ and $\alpha=2$ the convergence highly depends on the start point. For $\alpha=0.001$ and $\alpha=0.99$ the trajectories move steadily into the minimum and show stable convergence. At $\alpha=0.5$, the surface produces monotonic descent to the global minimum. In summary, moderate $\alpha \in (0,1)$ values give robust convergence, whereas extreme values make the landscape more sensitive to initialization. Figure~\ref{fig:bf-alpha} illustrates the sensitivity of pattern recognition with respect to the boundary factor ($BF$) and the divergence parameter $\alpha$. At $BF=0$, Chem-NMF reduces to the baseline $\alpha$-NMF. Adding $BF$ improves performance, particularly at intermediate values of $\alpha$. 
\begin{figure}[H]
    \centering
    \includegraphics[width=0.7\textwidth]{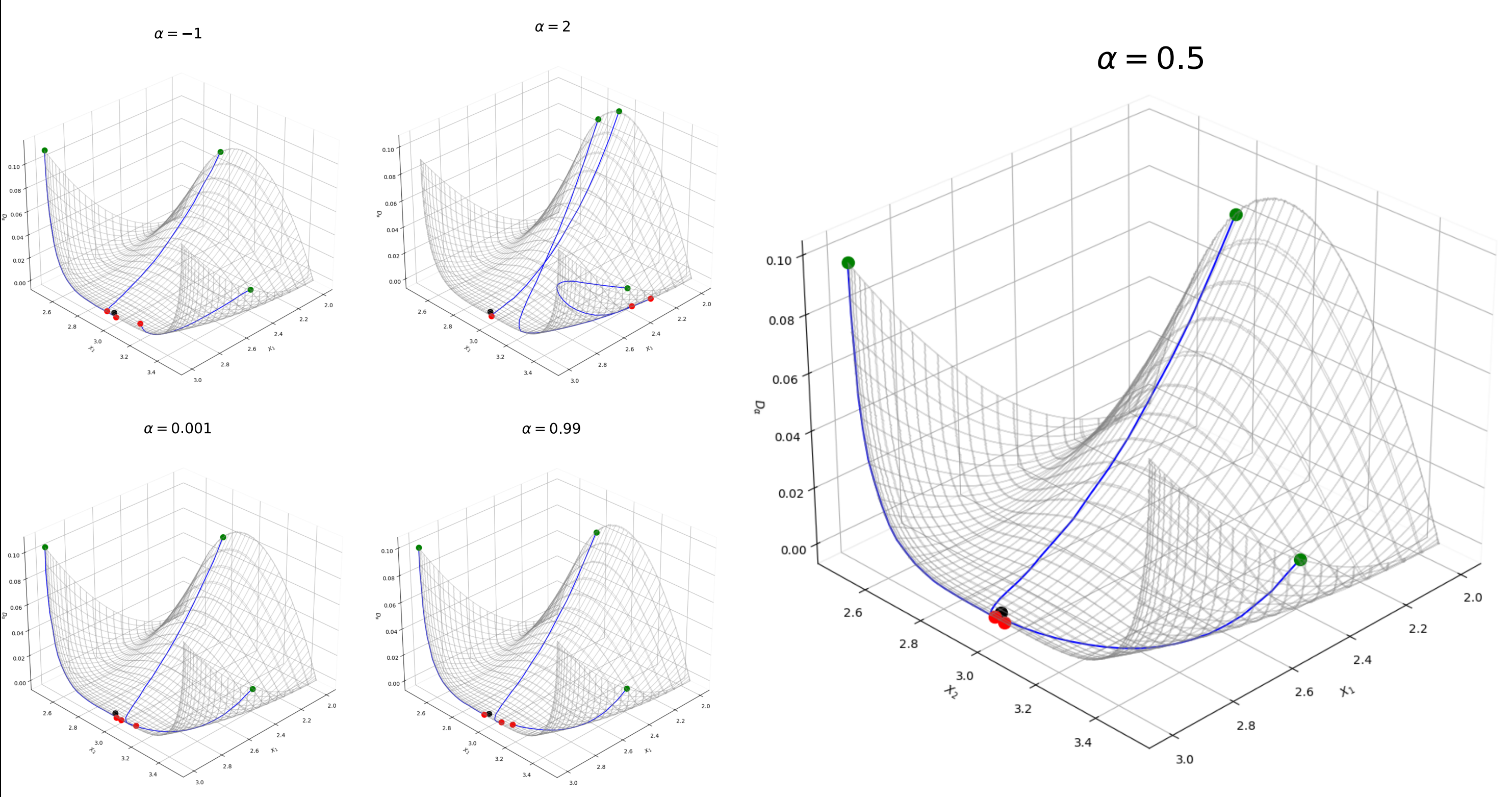}
    \captionsetup{font=footnotesize}
    \caption{Effect of $\alpha$ value on the optimization landscape. Each subplot shows the trajectory for a specific $\alpha$: green points indicate the initialization, red points denote the final optimized solutions, and the black point marks the desired global minimum.}
    \label{fig:alpha}
\end{figure}

\begin{figure}[H]
    \centering
    \includegraphics[width=0.5\textwidth]{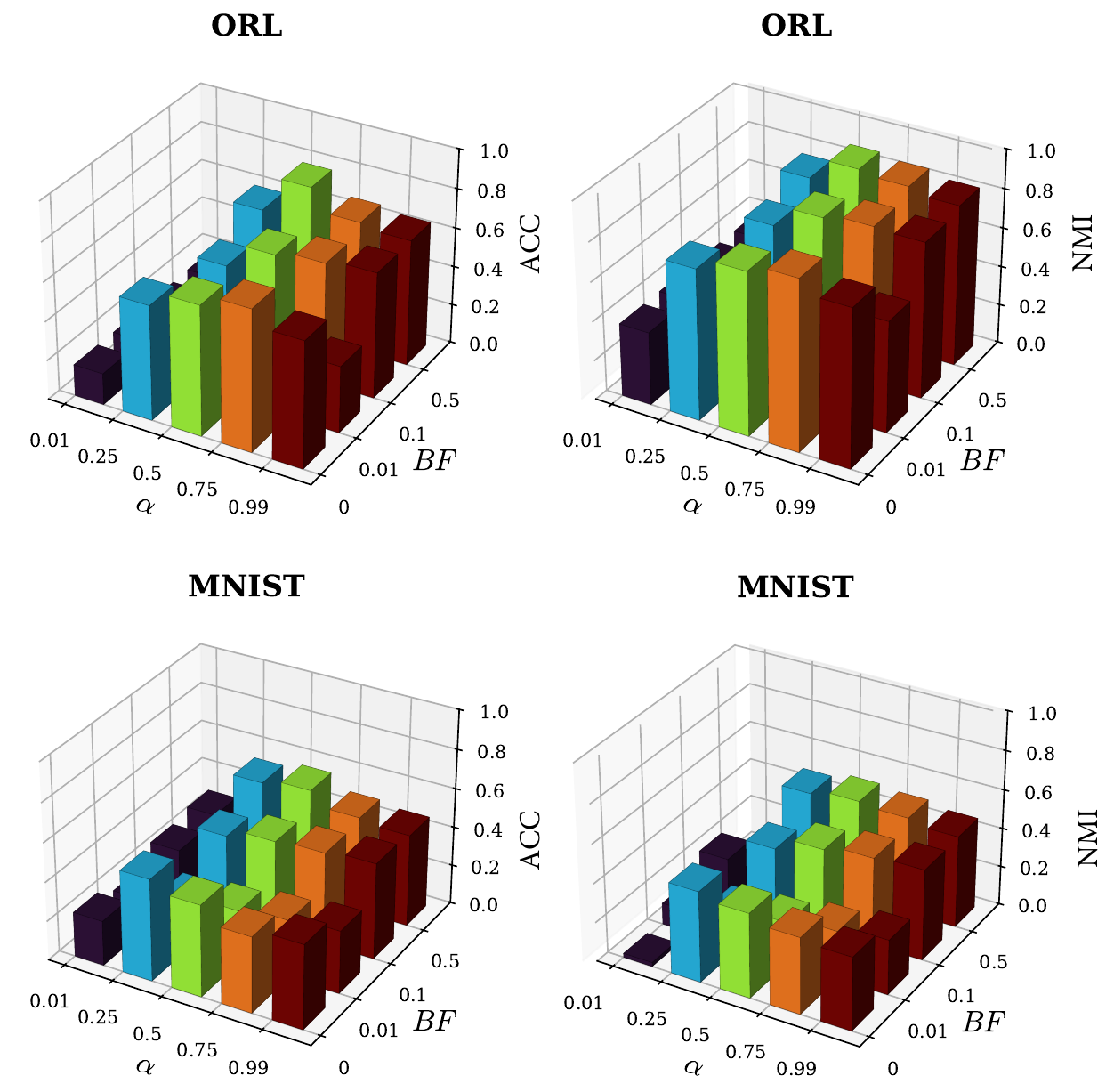}
    \captionsetup{font=footnotesize}
    \caption{Effect of $BF$ and $\alpha$ parameter on pattern recognition performance for ORL and MNIST datasets.}
    \label{fig:bf-alpha}
\end{figure}

\subsection{Robustness in Noisy Conditions}
We tested Regular NMF, \(\alpha\)-NMF, and Chem-NMF on ORL and MNIST image recognition datasets to evaluate clustering performance under different noise conditions (See Table~\ref{tab:acc-nmi} and Table~\ref{tab:mnist-acc-nmi} in Appendix). We added Gaussian noise at 5--30 dB to measure robustness. As shown in Fig.~\ref{fig:noise}, \(\alpha\)-NMF achieved higher NMI scores in the low-noise settings (5--10 dB), but its performance dropped more sharply as noise increased. Chem-NMF maintained higher scores at clean and high noise levels (20--30 dB), showing greater robustness to noise. Regular NMF consistently had the lowest values across both metrics. The divergence parameter \(\alpha\) strongly affects clustering accuracy. Mid-range values gave higher ACC and NMI, while very small or large values performed worse. Small \(\alpha\) tends to overfit noise, while large \(\alpha\) loses fine structure, so the middle values provide a balance. 

\begin{figure}[H]
    \centering
    \includegraphics[width=0.9\textwidth]{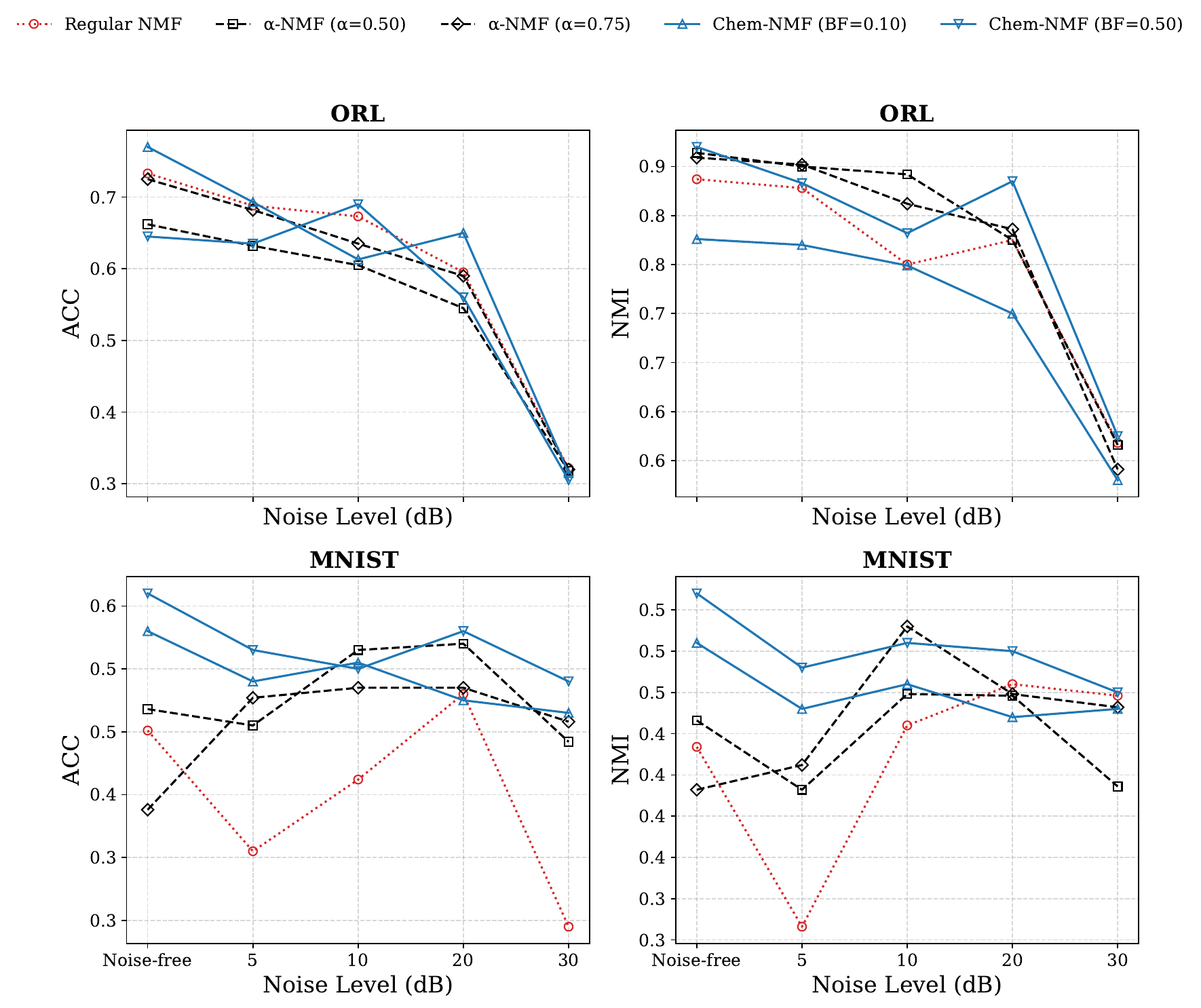}
    \captionsetup{font=footnotesize}
    \caption{Clustering performance of Regular NMF, $\alpha$-NMF, and Chem-NMF on ORL and MNIST datasets under different Gaussian 
    noise levels.}
    \label{fig:noise}
\end{figure}

\subsection{Numerical Convergence Analysis}
Fig.~\ref{fig:layer} illustrates the normalized training loss for a multi-layer $\alpha$-NMF run under different bounding factors ($BF$). Within each layer, the loss decreases and then flattens as updates approach a stationary point. When $BF=0$, the behaviour is equivalent to plain $\alpha$-NMF with random initialization. This setting explores aggressively, but it also causes sharp overshoots at layer boundaries and can trap the algorithm in higher local minima. At the other extreme, $BF=1$ enforces strict continuity across layers, which heavily bounds both initialization and update steps. While this avoids overshoot, it prevents sufficient exploration, and the algorithm may get stuck in suboptimal basins. Intermediate values $BF$ $\in(0,1)$ strike a balance between exploration and exploitation. They reduce the energy gap between successive layer minima while still allowing enough freedom to escape shallow plateaus. This balance yields smoother convergence and consistently lower final losses.

\begin{figure}[H]
    \centering
    \includegraphics[width=\textwidth]{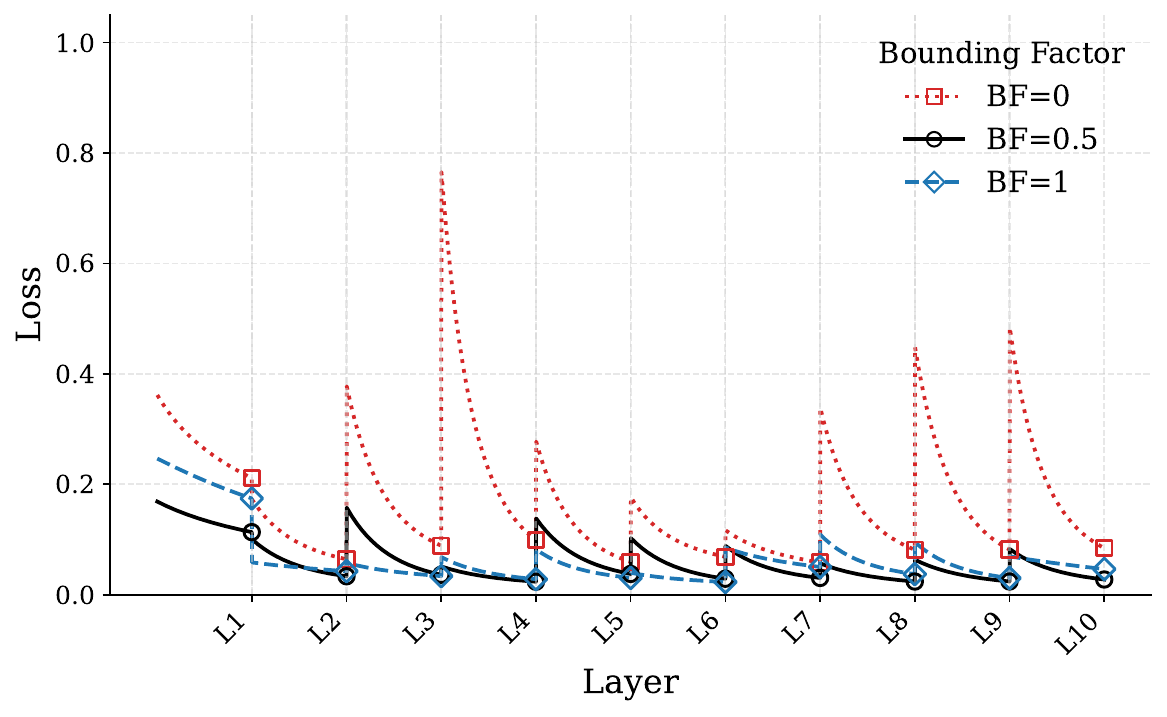}
    \captionsetup{font=footnotesize}
    \caption{Loss per iteration across layers for Chem-NMF with different bounding factors. 
    Markers denote the final $\alpha$-divergence value attained at the end of each layer, 
    representing the local optimum reached before re-initialization in the next layer.}
    \label{fig:layer}
\end{figure}

\subsection{Clinical Application: Clustering Cardiovascular Sounds}

We evaluate the utility of Chem-NMF in clinical applications by performing unsupervised clustering on lung and heart sound datasets. We transform the recordings into time–frequency spectrograms, factorize the data into a low-rank representation, and cluster data using K-means, Gaussian mixture models (GMM), agglomerative clustering, and spectral clustering. For the ablation study, we compare clustering performance without and with NMF feature extraction (Figure~\ref{fig:ablation}). The results demonstrate that NMF improves clustering performance.

\begin{figure}[H]
    \centering
    \includegraphics[width=\textwidth]{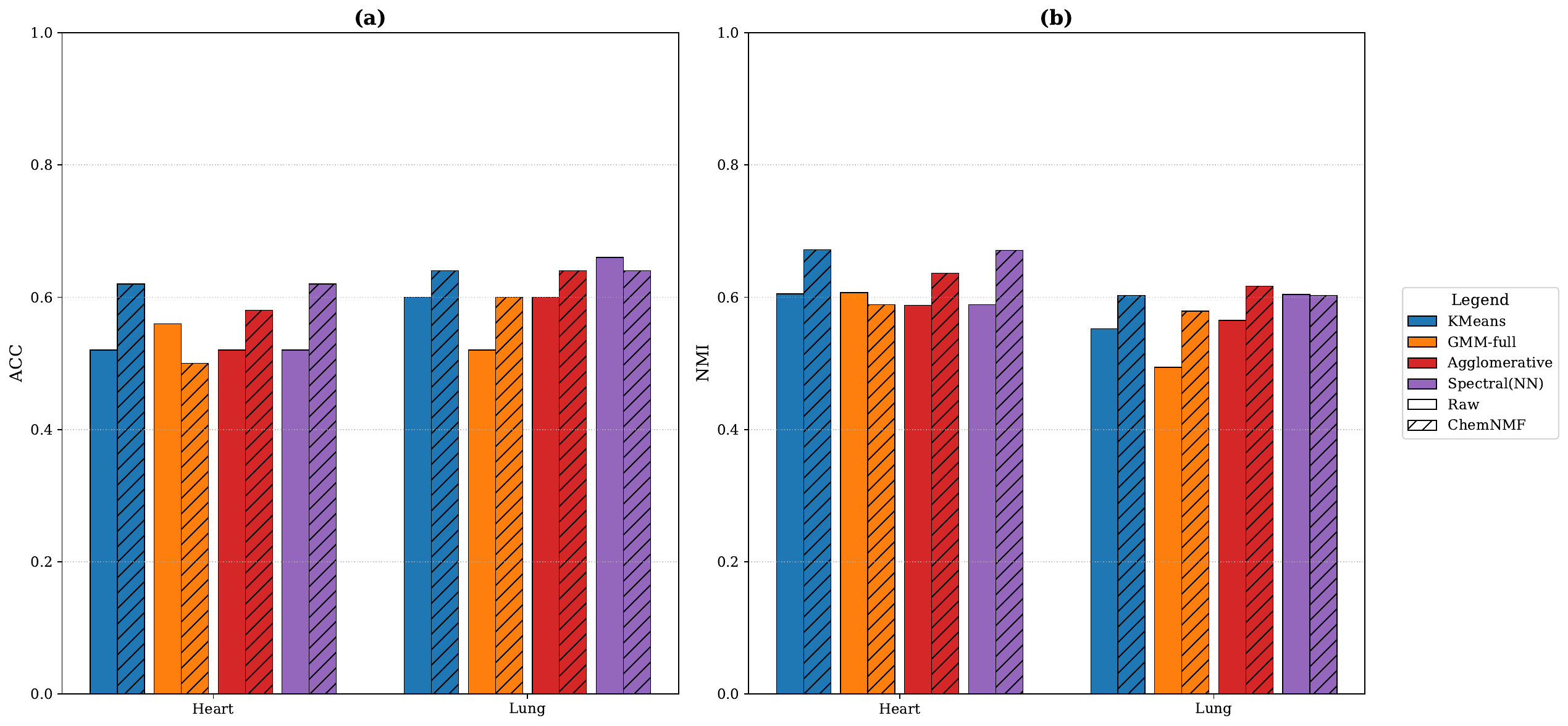}
    \captionsetup{font=footnotesize}
    \caption{Ablation study on the effect of Chem-NMF feature extraction on clustering performance of cardiovascular sounds based on: \textbf{(a)} ACC and \textbf{(b)} NMI measures.}
    \label{fig:ablation}
\end{figure}

\subsection{Comparison Performance}

Table~\ref{tab:compare} compares Chem-NMF with several recent NMF variants on ORL dataset. Across the evaluated datasets, Chem-NMF reaches an accuracy of 78\%, which represents 
an average improvement of 11\% $\pm$ 7\% over recent baselines. This indicates that while existing models contribute important advances, 
the chemical reaction–inspired formulation provides additional gains in clustering performance.

\begin{table}[H]
\centering
\footnotesize
\captionsetup{font=footnotesize}
\caption{Image Recognition performance of NMF algorithms on ORL dataset}
\begin{tabular}{p{0.15\linewidth} p{0.15\linewidth} p{0.15\linewidth} p{0.45\linewidth}}
\toprule
\textbf{[Ref]} & \textbf{Method} & \textbf{Accuracy (\%)} & \textbf{Description} \\
\midrule
This work & \textbf{Chem-NMF} & 78 & Chemical reaction-inspired \\
\midrule
\cite{Wan2026} & RLNMF-SP & 76 & Robust locality-regularized \\
\cite{9353266} & DR-NMF & 59 & Distributionally robust multi-objective \\
\cite{Barkhoda2026} & iDRNMF & 60 & Instance-wise distributionally robust \\
\cite{Salahian2023} & DAN-NMF & 58 & Deep autoencoder \\
\cite{cai2010graph} & GNMF & 70 & Graph-regularized \\
\cite{lu2017learning} & LRNF & 71 & Low-rank \\
\cite{he2020low} & LNMFS & 72 & Low-rank NMF on a Stiefel manifold \\
\cite{guo2021double} & DMR-NMF & 74 & Double manifolds regularized \\
\bottomrule
\end{tabular}
\label{tab:compare}
\end{table}

\newpage

\section{Discussion}

The findings of this work highlight the advantages of analyzing multi-layer $\alpha$-divergence NMF through an energy-based perspective. By introducing Chem-NMF with a bounding factor, we demonstrated that convergence can be stabilized while escaping poor local minima. This supports the theoretical analysis showing that multi-layer architectures reduce the probability of becoming trapped in suboptimal basins, though at the expense of slower convergence. The bounding factor plays a role analogous to a chemical catalyst. It regulates the initialization across layers, which leads to lowering the effective activation barrier, and balancing exploration and exploitation during optimization. Experimental evaluations on both image and biomedical audio datasets confirmed these theoretical analyses. The chemical analogy provides a useful framework for interpreting these results. Just as reactants traverse sequential activation barriers to reach stable products, Chem-NMF progresses across layers that gradually reduce divergence and improve stability. The connection between thermodynamic principles and optimization dynamics offers an intuitive and rigorous foundation for designing more reliable NMF algorithms. Nonetheless, limitations remain. The datasets employed may not fully reflect real-world variability. Additional validation on larger and more heterogeneous datasets is needed to assess scalability and clinical applicability. Furthermore, the multi-layer structure increases computational cost, motivating future work on adaptive strategies that adjust the bounding factor or depth dynamically. Finally, extending the theoretical framework to stochastic thermodynamics or quantum-inspired models could broaden the NMF application beyond clustering.

\section{Conclusion}

In this paper, we introduced Chem-NMF, a multi-layer $\alpha$-divergence NMF framework inspired by energy barriers in chemical reactions. By incorporating a bounding factor analogous to a chemical catalyst, the method stabilizes convergence, reduces overshoot, and improves clustering performance compared to Regular NMF and plain $\alpha$-NMF. Theoretical analysis confirmed a lower probability of staying in local minima, while experiments on image and biomedical datasets demonstrated clustering accuracy. These results establish Chem-NMF as a promising extension of NMF with practical potential across diverse applications.

\section*{Dataset Availability and Source Codes}
The Python scripts are available at \url{https://github.com/Torabiy/ChemNMF}. The dataset is available at \url{https://github.com/Torabiy/HLS-CMDS}.

\pagebreak
\begingroup
\setlength\bibitemsep{1pt} 
\renewcommand{\bibfont}{\fontsize{8pt}{11pt}\selectfont}  
\printbibliography
\endgroup

\pagebreak
\section*{Appendix}
\renewcommand{\thetable}{G\arabic{table}}
\setcounter{table}{0} 
\renewcommand{\thesubsection}{\Alph{subsection}} 

\subsection{Theorem 3.1}
\begin{proof}

We differentiate the cost function $D_{\alpha}(\mathbf{Y} \parallel \mathbf{A}\mathbf{X})$ with respect to $x_{jt}$:

\begin{equation}
\frac{\partial D}{\partial x_{jt}} = \frac{1}{\alpha} \sum_{i} a_{ij} \left[ 1 - \left( \frac{y_{it}}{[\mathbf{A}\mathbf{X}]_{it}} \right)^{\alpha} \right].
\end{equation}

To derive a multiplicative update rule, we employ a projected (transformed) gradient descent approach:

\begin{equation}
\Phi(x_{jt}) \leftarrow \Phi(x_{jt}) - \eta_{jt} \frac{\partial D}{\partial \Phi(x_{jt})},
\label{eq:grad}
\end{equation}

where we define the transformation function $\Phi(x) = x^{\alpha}$, and choose the learning rate as:

\begin{equation}
\eta_{jt} = \frac{\alpha^2 \Phi(x_{jt})}{x_{jt}^{1-\alpha} \sum_{i} a_{ij}}.
\label{eq:eta}
\end{equation}

Applying this transformation and using the chain rule, we obtain:

\begin{equation}
\frac{\partial D}{\partial \Phi(x_{jt})} 
= \frac{\partial D}{\partial x_{jt}} \cdot \frac{\partial x_{jt}}{\partial \Phi(x_{jt})} 
= \frac{1}{\alpha} \sum_{i} a_{ij} \left[ 1 - \left( \frac{y_{it}}{[\mathbf{A}\mathbf{X}]_{it}} \right)^{\alpha} \right] \cdot \frac{1}{\alpha x_{jt}^{\alpha-1}}.
\label{eq:devgrad}
\end{equation}

Since $\Phi(x_{jt}) = x_{jt}^{\alpha}$, we substitute ~(\ref{eq:eta}) and (\ref{eq:devgrad}) into ~(\ref{eq:grad}), yielding:
\allowdisplaybreaks
\begin{align}
x_{jt} & \leftarrow \Phi^{-1} \left( \Phi(x_{jt}) - \eta_{jt} \frac{\partial D}{\partial \Phi(x_{jt})} \right) \nonumber \\
&\leftarrow \left( x_{jt}^{\alpha} - \frac{\alpha^2 x_{jt}^{\alpha}}{x_{jt}^{1-\alpha} \sum_{i} a_{ij}} \cdot \frac{\partial D}{\partial \Phi(x_{jt})} \right)^{1/\alpha} \nonumber \\
&\leftarrow \left( x_{jt}^{\alpha} - \frac{\alpha^2 x_{jt}^{\alpha}}{x_{jt}^{1-\alpha} \sum_{i} a_{ij}} \cdot \frac{1}{\alpha} \sum_{i} a_{ij} \left[ 1 - \left( \frac{y_{it}}{[\mathbf{A}\mathbf{X}]_{it}} \right)^{\alpha} \right] \cdot \frac{1}{\alpha x_{jt}^{\alpha-1}} \right)^{1/\alpha} \nonumber \\
&\leftarrow \Bigg( x_{jt}^{\alpha} - \frac{\alpha^2 x_{jt}^{\alpha}}{x_{jt}^{1-\alpha} \sum_{i=1}^{I} a_{ij}} \cdot \frac{1}{\alpha} \sum_{i} a_{ij} \cdot \frac{1}{\alpha x_{jt}^{\alpha-1}} \nonumber \\ 
&\quad + \frac{\alpha^2 x_{jt}^{\alpha}}{x_{jt}^{1-\alpha} \sum_{i} a_{ij}} \cdot \frac{1}{\alpha} \sum_{i} a_{ij} \left( \frac{y_{it}}{[\mathbf{A}\mathbf{X}]_{it}} \right)^{\alpha} \cdot \frac{1}{\alpha x_{jt}^{\alpha-1}} \Bigg)^{1/\alpha} \nonumber \\ 
&\leftarrow \left( x_{jt}^{\alpha} - x_{jt}^{\alpha} + \frac{x_{jt}^{\alpha}}{\sum_{i=1}^{I} a_{ij}} \cdot \sum_{i} a_{ij} \left( \frac{y_{it}}{[\mathbf{A}\mathbf{X}]_{it}} \right)^{\alpha} \right)^{1/\alpha} \nonumber \\
&\leftarrow \left( \frac{x_{jt}^{\alpha}}{\sum_{i} a_{ij}} \cdot \sum_{i} a_{ij} \left( \frac{y_{it}}{[\mathbf{A}\mathbf{X}]_{it}} \right)^{\alpha} \right)^{1/\alpha}. \nonumber
\end{align}

\begin{equation}
    x_{jt} \leftarrow x_{jt} \left( \frac{\sum\limits_{i} a_{ij} \left( \frac{y_{it}}{[\mathbf{A}\mathbf{X}]_{it}} \right)^{\alpha}}{\sum\limits_{i} a_{ij}} \right)^{\frac{1}{\alpha}}.
\end{equation}

Similarly, we derive the update rule for $a_{ij}$ as:  

\begin{equation}
a_{ij} \leftarrow a_{ij} \left( \frac{\sum\limits_{t=1}^{T} x_{jt} \left( \frac{y_{it}}{[\mathbf{A}\mathbf{X}]_{it}} \right)^{\alpha}}{\sum\limits_{t=1}^{T} x_{jt}} \right)^{\frac{1}{\alpha}}. 
\end{equation}

\qedhere \text{ QED.}
\end{proof}

\subsection{Lemma 3.1}
\begin{proof} 
We have two conditions:

\textbf{(i) Identity Condition:} \textit{Setting $\mathbf{X}' = \mathbf{X}$ in the auxiliary function $G(\mathbf{X}, \mathbf{X}')$ recovers the original $F(\mathbf{X})$, such that  $G(\mathbf{X}, \mathbf{X}) = F(\mathbf{X})$.}

Setting $\mathbf{X}' = \mathbf{X}$, we simplify $\zeta_{itj}$:
\[
\zeta_{itj} = \frac{a_{ij} x_{jt}}{\sum_{j=1}^{J} a_{ij} x_{jt}} 
= \frac{a_{ij} x_{jt}}{[\mathbf{A}\mathbf{X}]_{it}}.
\]

Substituting $\zeta_{itj}$ into $G(\mathbf{X}, \mathbf{X})$ and simplifying, we get:

\begin{align}
G(\mathbf{X}, \mathbf{X}) 
&= \frac{1}{\alpha(\alpha-1)} \sum_{ijt} 
y_{it} \frac{a_{ij} x_{jt}}{[\mathbf{A}\mathbf{X}]_{it}} 
\left[ \left( \frac{[\mathbf{A}\mathbf{X}]_{it}}{y_{it}} \right)^{1-\alpha} 
+ (\alpha-1) \frac{[\mathbf{A}\mathbf{X}]_{it}}{y_{it}} - \alpha \right] \nonumber\\
&= \frac{1}{\alpha(\alpha-1)} \sum_{it} 
y_{it} \frac{\sum_{j=1}^{J} a_{ij} x_{jt}}{[\mathbf{A}\mathbf{X}]_{it}} 
\left[ \left( \frac{[\mathbf{A}\mathbf{X}]_{it}}{y_{it}} \right)^{1-\alpha} 
+ (\alpha-1) \frac{[\mathbf{A}\mathbf{X}]_{it}}{y_{it}} - \alpha \right] \nonumber\\
&= \frac{1}{\alpha(\alpha-1)} \sum_{it} 
y_{it} \frac{[\mathbf{A}\mathbf{X}]_{it}}{[\mathbf{A}\mathbf{X}]_{it}} 
\left[ \left( \frac{[\mathbf{A}\mathbf{X}]_{it}}{y_{it}} \right)^{1-\alpha} 
+ (\alpha-1) \frac{[\mathbf{A}\mathbf{X}]_{it}}{y_{it}} - \alpha \right] \nonumber\\
&= \frac{1}{\alpha(\alpha-1)} \sum_{it} 
\Big( [\mathbf{A}\mathbf{X}]_{it}^{1-\alpha} y_{it}^\alpha 
+ (\alpha-1) [\mathbf{A}\mathbf{X}]_{it}^\alpha - \alpha y_{it} \Big) 
= F(\mathbf{X}).
\end{align}

\qedhere \text{ QED.}

\vspace{1em}
\textbf{(ii) Upper Bound Condition:}  
\textit{The auxiliary function $G(\mathbf{X}, \mathbf{X}')$ provides an upper bound on $F(\mathbf{X})$, such that $G(\mathbf{X}, \mathbf{X}') \geq F(\mathbf{X})$.}

\begin{definition}[Jensen’s Inequality]
Let $f(z)$ be a convex function. Then, for any weights $w_j \geq 0$ such that $\sum_{j} w_j = 1$, we have:
\begin{equation}
    f\left(\sum_{j} w_j z_j\right) \leq \sum_{j} w_j f(z_j).
\end{equation}
\end{definition}

We consider the function associated with the $\alpha$-divergence:
\begin{equation}
    f(z) = \frac{1}{\alpha(\alpha-1)} \left[ z^{1-\alpha} + (\alpha -1)z - \alpha \right].
\end{equation}

Its first derivative is:
\begin{equation}
    f'(z) = \frac{1}{\alpha(\alpha-1)} \left[ (1-\alpha)z^{-\alpha} + (\alpha-1) \right].
\end{equation}

Differentiating again, we obtain:
\begin{equation}
    f''(z) = \frac{1}{\alpha(\alpha-1)} \left[ -\alpha(1-\alpha)z^{-\alpha-1} \right].
\end{equation}

Rewriting this:
\begin{equation}
    f''(z) = z^{-\alpha-1}.
\end{equation}

Since $z^{-\alpha-1} \geq 0$ for $z > 0$, we conclude that $f(z)$ is convex.  
Now, applying Jensen’s inequality, we obtain:
\begin{equation}
    f\left( \sum_{j}\frac{a_{ij} x_{jt}}{y_{it}} \right) 
    \leq \sum_{j} \zeta_{itj} f\left(\frac{a_{ij} x_{jt}}{y_{it}\zeta_{itj}} \right),
\end{equation}
where the weights $\zeta_{itj}$ are defined as:
\begin{equation}
    \zeta_{itj} = \frac{a_{ij} x'_{jt}}{\sum_{j=1}^{J} a_{ij} x'_{jt}}, 
    \quad \sum_{j} \zeta_{itj} = 1, \quad \zeta_{itj} \geq 0.
\end{equation}

Multiplying both sides by $y_{it}$ and summing over all $i$ and $t$, we get:
\begin{equation}
   F(\mathbf{X}) = \sum_{it} y_{it} f\left( \sum_{j}\frac{a_{ij} x_{jt}}{y_{it}} \right) 
   \leq \sum_{itj} y_{it} \zeta_{itj} f\left(\frac{a_{ij} x_{jt}}{y_{it}\zeta_{itj}} \right) 
   = G(\mathbf{X}, \mathbf{X}').
\end{equation}
\qedhere $\square$ \text{ QED.}
\end{proof}

\subsection{Theorem 3.2}
\begin{proof} 

Consider the function associated with the $\alpha$-divergence:

\begin{equation}
    f(z) = \frac{1}{\alpha(\alpha-1)} \left[ z^{1-\alpha} + (\alpha -1)z - \alpha \right].
\end{equation}

Its first derivative is:

\begin{equation}
    f'(z) = \frac{1}{\alpha(\alpha-1)} \left[ (1-\alpha)z^{-\alpha} + (\alpha-1) \right].
    \label{eq:firstderivative}
\end{equation}

Rewriting $F(\mathbf{X})$ and $G(\mathbf{X},\mathbf{X}')$ as:

\begin{equation}
   F(\mathbf{X})= \sum_{it} y_{it} f\left( \sum_{j}\frac{ a_{ij} x_{jt} }{y_{it}}\right).
\end{equation}

\begin{equation}
   G(\mathbf{X},\mathbf{X}')=\sum_{itj} y_{it} \zeta_{itj} f\left(\frac{a_{ij} x_{jt}}{y_{it}\zeta_{itj}} \right).
\end{equation}

We minimize $G(\mathbf{X}, \mathbf{X}')$ by setting the gradient to zero:

\begin{equation}
    \frac{\partial G(\mathbf{X}, \mathbf{X}')}{\partial x_{jt}} = \sum_{i} y_{it} \zeta_{itj} f'\left(\frac{a_{ij} x_{jt}}{y_{it}\zeta_{itj}}\right) \cdot \frac{\partial}{\partial x_{jt}} \left( \frac{a_{ij} x_{jt}}{y_{it}\zeta_{itj}} \right)=0.
\end{equation}

Since \( \frac{\partial}{\partial x_{jt}} \left( \frac{a_{ij} x_{jt}}{y_{it}\zeta_{itj}} \right) = \frac{a_{ij}}{y_{it}\zeta_{itj}} \), we get:

\begin{equation}
    \frac{\partial G(\mathbf{X}, \mathbf{X}')}{\partial x_{jt}} = \sum_{i} y_{it} \zeta_{itj} f'\left(\frac{a_{ij} x_{jt}}{y_{it}\zeta_{itj}}\right) \frac{a_{ij}}{y_{it}\zeta_{itj}}=0.
\end{equation}

Substituting \( f'(z) \) from ~(\ref{eq:firstderivative}) into the expression:

\begin{align}
    \frac{\partial G(\mathbf{X}, \mathbf{X}')}{\partial x_{jt}} = & \sum_{i} y_{it}\zeta_{itj} \cdot\frac{1}{\alpha(\alpha-1)} \left[ (1-\alpha) \left( \frac{a_{ij} x_{jt}}{y_{it}\zeta_{itj}} \right)^{-\alpha} + (\alpha-1) \right] \frac{a_{ij}}{y_{it}\zeta_{itj}} \nonumber \\
    & = \frac{1}{\alpha}\sum_{i}  a_{ij}\left[ 1- \left( \frac{a_{ij} x_{jt}}{y_{it}\zeta_{itj}} \right)^{-\alpha} \right] =0.
\end{align}

Rearranging the equation for $\alpha \neq 0$:

\begin{equation}
    \sum_{i} a_{ij} = \sum_{i} a_{ij} \left( \frac{a_{ij} x_{jt}}{y_{it} \zeta_{itj}} \right)^{-\alpha}= \sum_{i} a_{ij} \left( \frac{y_{it} \zeta_{itj}}{a_{ij} x_{jt}} \right)^{\alpha}.
\end{equation}

Dividing both sides by \( \sum_{i}a_{ij} \):

\begin{equation}
    1 = \frac{\sum_{i} a_{ij} \left( \frac{y_{it} \zeta_{itj}}{a_{ij} x_{jt}} \right)^{\alpha}}{\sum_{i} a_{ij}}.
\end{equation}

Substituting \( \zeta_{itj} \) from ~(\ref{eq:zeta}) into the expression:

\begin{equation}
    1 = \frac{\sum_{i} a_{ij} \left( \frac{y_{it}a_{ij}x'_{jt}}{a_{ij}x_{jt}\sum_{i} a_{ij} x'_{jt}} \right)^{\alpha}}{\sum_{i}a_{ij}} = \frac{\sum_{i} a_{ij} \left( \frac{y_{it}}{\sum_{i} a_{ij}x'_{jt}} \right)^{\alpha}\cdot \left( \frac{x'_{jt}}{x_{jt}} \right)^{\alpha}}{\sum_{i}a_{ij}},
\end{equation}

which leads to:

\begin{equation}
    \left(\frac{x_{jt}}{x'_{jt}}\right) = \left[ \frac{\sum\limits_{i} a_{ij} \left( \frac {y_{it}} {\sum\limits_{i} a_{ij}x'_{jt}}\right)^{\alpha}}{\sum\limits_{i} a_{ij}} \right]^{1/\alpha},
\end{equation}

which suggests the following update rule for \( x_{jt} \):

\begin{equation}
    x_{jt} \leftarrow x_{jt} \left( \frac{\sum\limits_{i} a_{ij} \left( \frac{y_{it}}{[\mathbf{A}\mathbf{X}]_{it}} \right)^{\alpha}}{\sum\limits_{i} a_{ij}} \right)^{\frac{1}{\alpha}}.
\end{equation}

Since \( G(\mathbf{X}, \mathbf{X}') \) is an auxiliary function for \( F(\mathbf{X}) \), minimizing \( G(\mathbf{X}, \mathbf{X}') \) at each step ensures that \( F(\mathbf{X}) \) is non-increasing: 

\begin{equation}  
    F(\mathbf{X}^{(t+1)}) \leq G(\mathbf{X}^{(t+1)}, \mathbf{X}^{(t)}) \leq G(\mathbf{X}^{(t)}, \mathbf{X}^{(t)}) = F(\mathbf{X}^{(t)}).
\end{equation}  

\qedhere \text{ QED.}

\end{proof}

\subsection{Lemma 4.1}
\begin{proof}
Each layer solves the following minimization problem:

\begin{equation}
    (\mathbf{A}^{(l)}, \mathbf{X}^{(l)}) = \arg \min_{\mathbf{A}, \mathbf{X}} D_{\alpha}(\mathbf{X}^{(l-1)} \parallel \mathbf{A}\mathbf{X}).
\end{equation}

This ensures:

\begin{equation}
    D_{\alpha}(\mathbf{X}^{(l-1)} \parallel \mathbf{A}^{(l)} \mathbf{X}^{(l)}) \leq D_{\alpha}(\mathbf{X}^{(l-1)} \parallel \mathbf{A}^{(l-1)} \mathbf{X}^{(l-1)}).
    \label{eq:this}
\end{equation}

Applying the non-increasing property in ~(\ref{eq:Gfunc}) to two consecutive layers, we obtain:

\begin{equation}
     D_{\alpha}(\mathbf{X}^{(l-1)} \parallel \mathbf{A}^{(l-1)} \mathbf{X}^{(l-1)})\leq D_{\alpha}(\mathbf{X}^{(l-2)} \parallel \mathbf{A}^{(l-1)} \mathbf{X}^{(l-1)}).
     \label{eq:that}
\end{equation}

Combining ~(\ref{eq:this}) and (\ref{eq:that}), we derive:

\begin{equation}
    D_{\alpha}(\mathbf{X}^{(l-1)} \parallel \mathbf{A}^{(l)} \mathbf{X}^{(l)}) \leq D_{\alpha}(\mathbf{X}^{(l-1)} \parallel \mathbf{A}^{(l-1)} \mathbf{X}^{(l-1)})\leq D_{\alpha}(\mathbf{X}^{(l-2)} \parallel \mathbf{A}^{(l-1)} \mathbf{X}^{(l-1)}).
\end{equation}

Thus, by definition,

\begin{equation}
    D_l \leq D_{l-1}.
\end{equation}
\qedhere \text{ QED.}
\end{proof}

\subsection{Lemma 4.2}
\begin{proof}
Since $D_l, M_l \ge 0$ are non-increasing and lower bounded, the sequences $M_l$ and $D_l$ converge to finite limits $M_\infty$ and $D_\infty$, respectively.
\begin{equation}
\lim_{l\to\infty} M_l = M_\infty,\qquad \lim_{l\to\infty} D_l = D_\infty.
\end{equation}
For any $\varepsilon>0,\;\exists\,N_M,N_D\in\mathbb{N}$ such that:
\begin{equation}
l\ge N_M \Rightarrow |M_l-M_\infty|<\tfrac{\varepsilon}{2},\qquad l\ge N_D \Rightarrow |D_l-D_\infty|<\tfrac{\varepsilon}{2}.
\end{equation}
Let $N_\xi=\max\{N_M,\;N_D\}$. Then for all $l\ge N_\xi$:
\begin{equation}
\begin {aligned}
&|\xi_l-(M_\infty-D_\infty)|
=|(M_l-M_\infty)-(D_{l-1}-D_\infty)| \\
&\le |M_l-M_\infty|+|D_{l-1}-D_\infty|
<\varepsilon.
\end {aligned}
\end{equation}
Which implies:
\begin{equation}
\lim_{l\to\infty}\xi_l = M_\infty - D_\infty. 
\end{equation}
By continuity property of Eq.~(12) we have:
\begin{equation}
\lim_{l\to\infty} P_l
=\lim_{l\to\infty}\tfrac{1}{Z}e^{-\beta\xi_l}
=\tfrac{1}{Z}e^{-\beta\lim_{l\to\infty}\xi_l}
=\tfrac{1}{Z}e^{-\beta(M_\infty-D_\infty)}
=P_\infty.
\end{equation}
\qedhere \text{ QED.}
\end{proof}

\subsection{Lemma 4.3}
\begin{proof}
At layer $l$, the escape probability is $P_l$, defined by Eq.~(12). This implies:
\begin{equation}
\mathbb{P}(\text{no escape at layer }l \mid S_{l-1}) = 1 - P_l.
\end{equation}
For $l=1$,
\begin{equation}
\mathbb{P}(S_1)= 1- P_l \vert_{l=1}=1-P_1.
\end{equation}
The law of conditional probability states:
\begin{equation}
\mathbb{P}(S_l)=\mathbb{P}(S_{l-1})\cdot \mathbb{P}(\text{no escape at layer }l \mid S_{l-1})=\mathbb{P}(S_{l-1})\cdot(1 - P_l).
\end{equation}
Assume for some $m \ge 1$:
\begin{equation}
\mathbb{P}(S_m)=\prod_{j=1}^{m}(1-P_j).
\end{equation}
Thus:
\begin{equation}
\mathbb{P}(S_{m+1})
=\mathbb{P}(S_m)(1-P_{m+1})
=\Big(\prod_{j=1}^{m}(1-P_j)\Big)(1-P_{m+1})
=\prod_{j=1}^{m+1}(1-P_j).
\end{equation}
By induction, the claim holds for all $n \ge 1$. Therefore:
\begin{equation}
   \mathbb{P}(S_n)=\prod_{l=1}^{n}(1-P_l). 
\end{equation}
\qedhere \text{ QED.}
\end{proof}

\subsection{Tables}

\begin{table}[H]
\centering
\footnotesize
\captionsetup{font=footnotesize}
\caption{Clustering performance on the ORL Dataset under different noise levels.}
\label{tab:acc-nmi}
\begin{tabular}{llllccccc}
\toprule
\textbf{Metric} & \textbf{Method} & \boldmath$BF$ & \multicolumn{1}{c}{\boldmath$\alpha$}
 & \textbf{Noise-free} & \textbf{5 dB} & \textbf{10 dB} & \textbf{20 dB} & \textbf{30 dB} \\
\toprule
\multirow{25}{*}{\rotatebox{90}{ACC}}
 & Regular NMF &  &            & 0.733 & 0.688 & 0.673 & 0.595 & 0.322 \\
 \cmidrule(lr){2-9}
 & \multirow{5}{*}{\(\alpha\)-NMF} &
   & 0.01 & 0.160 & 0.158 & 0.159 & 0.161 & 0.157 \\
 &  &              & 0.25 & 0.585 & 0.287 & 0.560 & 0.468 & 0.545 \\
 &  &              & 0.50 & 0.662 & 0.632 & 0.605 & 0.545 & 0.318 \\
 &  &              & 0.75 & 0.720 & 0.682 & 0.682 & 0.570 & 0.352 \\
 &  &              & 0.99 & 0.642 & 0.642 & 0.585 & 0.460 & 0.263 \\
 \cmidrule(lr){2-9}
 & \multirow{15}{*}{Chem-NMF} 
   & \multirow{5}{*}{0.01}
   & 0.01 & 0.158 & 0.157 & 0.160 & 0.159 & 0.156 \\
 & &                   & 0.25 & 0.160 & 0.159 & 0.162 & 0.158 & 0.157 \\
 & &                   & 0.50 & 0.161 & 0.160 & 0.163 & 0.161 & 0.158 \\
 & &                   & 0.75 & 0.162 & 0.161 & 0.164 & 0.162 & 0.159 \\
 & &                   & 0.99 & 0.340 & 0.312 & 0.287 & 0.233 & 0.177 \\
 \cmidrule(lr){3-9}
 &  & \multirow{5}{*}{0.10}
   & 0.01 & 0.160 & 0.159 & 0.161 & 0.158 & 0.157 \\
 &  &                   & 0.25 & 0.460 & 0.430 & 0.420 & 0.347 & 0.233 \\
 &  &                   & 0.50 & 0.588 & 0.580 & 0.542 & 0.497 & 0.290 \\
 &  &                   & 0.75 & 0.618 & 0.595 & 0.620 & 0.547 & 0.345 \\
 &  &                   & 0.99 & 0.642 & 0.588 & 0.583 & 0.440 & 0.263 \\
 \cmidrule(lr){3-9}
 &  & \multirow{5}{*}{0.50}
   & 0.01 & 0.162 & 0.160 & 0.161 & 0.159 & 0.158 \\
 &  &                   & 0.25 & 0.593 & 0.532 & 0.547 & 0.455 & 0.302 \\
 &  &                   & 0.50 & \textbf{0.778} & \textbf{0.740} & 0.700 & \textbf{0.580} & 0.352 \\
 &  &                   & 0.75 & 0.667 & 0.672 & \textbf{0.713} & 0.590 & \textbf{0.357} \\
 &  &                   & 0.99 & 0.645 & 0.635 & 0.690 & 0.560 & 0.305 \\
\specialrule{1pt}{0pt}{0pt}
\addlinespace[0.5ex]
\multirow{25}{*}{\rotatebox{90}{NMI}}
 & Regular NMF &  &            & 0.837 & 0.828 & 0.750 & 0.775 & 0.568 \\
 \cmidrule(lr){2-9}
 & \multirow{5}{*}{\(\alpha\)-NMF} & 
   & 0.01 & 0.372 & 0.373 & 0.372 & 0.374 & 0.371 \\
 &  &              & 0.25 & 0.766 & 0.766 & 0.760 & 0.678 & 0.524 \\
 &  &              & 0.50 & 0.827 & \textbf{0.830} & 0.808 & 0.749 & 0.551 \\
 &  &              & 0.75 & 0.867 & 0.831 & \textbf{0.840} & 0.768 & 0.565 \\
 &  &              & 0.99 & 0.801 & 0.805 & 0.773 & 0.677 & 0.507 \\
 \cmidrule(lr){2-9}
 & \multirow{15}{*}{Chem-NMF} 
   & \multirow{5}{*}{0.01}
   & 0.01 & 0.372 & 0.371 & 0.373 & 0.374 & 0.370 \\
 & &                   & 0.25 & 0.373 & 0.372 & 0.374 & 0.373 & 0.371 \\
 & &                   & 0.50 & 0.374 & 0.373 & 0.375 & 0.374 & 0.372 \\
 & &                   & 0.75 & 0.375 & 0.374 & 0.376 & 0.375 & 0.373 \\
 & &                   & 0.99 & 0.566 & 0.549 & 0.526 & 0.464 & 0.409 \\
 \cmidrule(lr){3-9}
 &  & \multirow{5}{*}{0.10}
   & 0.01 & 0.372 & 0.372 & 0.372 & 0.372 & 0.372 \\
 &  &                   & 0.25 & 0.667 & 0.663 & 0.644 & 0.583 & 0.472 \\
 &  &                   & 0.50 & 0.776 & 0.770 & 0.749 & 0.700 & 0.530 \\
 &  &                   & 0.75 & 0.802 & 0.795 & 0.796 & 0.757 & 0.567 \\
 &  &                   & 0.99 & 0.794 & 0.780 & 0.768 & 0.650 & 0.509 \\
 \cmidrule(lr){3-9}
 &  & \multirow{5}{*}{0.50}
   & 0.01 & 0.372 & 0.367 & 0.376 & 0.371 & 0.366 \\
 &  &                   & 0.25 & 0.759 & 0.753 & 0.700 & 0.759 & 0.548 \\
 &  &                   & 0.50 & \textbf{0.870} & 0.833 & 0.782 & \textbf{0.835} & \textbf{0.575} \\
 &  &                   & 0.75 & 0.850 & 0.842 & 0.749 & 0.827 & 0.550 \\
 &  &                   & 0.99 & 0.821 & 0.844 & 0.782 & 0.831 & 0.548 \\
\bottomrule
\end{tabular}
\end{table}

\begin{table}[H]
\centering
\footnotesize
\captionsetup{font=footnotesize}
\caption{Clustering performance on the MNIST Dataset under different noise levels.}
\label{tab:mnist-acc-nmi}
\begin{tabular}{llllccccc}
\toprule
\textbf{Metric} & \textbf{Method} & \boldmath$BF$ & \multicolumn{1}{c}{\boldmath$\alpha$}
 & \textbf{Noise-free} & \textbf{5 dB} & \textbf{10 dB} & \textbf{20 dB} & \textbf{30 dB} \\
\toprule
\multirow{25}{*}{\rotatebox{90}{ACC}}
 & Regular NMF &  &            & 0.451 & 0.355 & 0.412 & 0.480 & 0.295 \\
 \cmidrule(lr){2-9}
 & \multirow{5}{*}{\(\alpha\)-NMF} & 
   & 0.01 & 0.228 & 0.192 & 0.186 & 0.180 & 0.198 \\
 &  &              & 0.25 & 0.520 & 0.502 & 0.513 & 0.463 & 0.473 \\
 &  &              & 0.50 & 0.468 & 0.455 & \textbf{0.515} & 0.520 & 0.442 \\
 &  &              & 0.75 & 0.388 & 0.477 & 0.485 & 0.485 & 0.458 \\
 &  &              & 0.99 & 0.430 & 0.442 & 0.442 & 0.505 & 0.475 \\
 \cmidrule(lr){2-9}
 & \multirow{15}{*}{Chem-NMF} 
   & \multirow{5}{*}{0.01}
   & 0.01 & 0.170 & 0.165 & 0.162 & 0.168 & 0.160 \\
 & &                   & 0.25 & 0.240 & 0.235 & 0.238 & 0.242 & 0.228 \\
 & &                   & 0.50 & 0.265 & 0.258 & 0.263 & 0.267 & 0.250 \\
 & &                   & 0.75 & 0.295 & 0.280 & 0.290 & 0.300 & 0.272 \\
 & &                   & 0.99 & 0.320 & 0.310 & 0.305 & 0.315 & 0.288 \\
 \cmidrule(lr){3-9}
 &  & \multirow{5}{*}{0.10}
   & 0.01 & 0.260 & 0.255 & 0.250 & 0.258 & 0.248 \\
 &  &                   & 0.25 & 0.410 & 0.395 & 0.405 & 0.400 & 0.375 \\
 &  &                   & 0.50 & 0.455 & 0.440 & 0.448 & 0.460 & 0.420 \\
 &  &                   & 0.75 & 0.470 & 0.455 & 0.462 & 0.470 & 0.430 \\
 &  &                   & 0.99 & 0.490 & 0.470 & 0.475 & 0.485 & 0.445 \\
 \cmidrule(lr){3-9}
 &  & \multirow{5}{*}{0.50}
   & 0.01 & 0.288 & 0.324 & 0.198 & 0.221 & 0.328 \\
 &  &                   & 0.25 & 0.530 & 0.490 & 0.505 & 0.475 & 0.465 \\
 &  &                   & 0.50 & \textbf{0.560} & \textbf{0.515} & 0.500 & \textbf{0.530} & \textbf{0.490} \\
 &  &                   & 0.75 & 0.490 & 0.495 & 0.500 & 0.505 & 0.470 \\
 &  &                   & 0.99 & 0.470 & 0.465 & 0.470 & 0.515 & 0.485 \\
\specialrule{1pt}{0pt}{0pt}
\addlinespace[0.5ex]
 \multirow{25}{*}{\rotatebox{90}{NMI}}& Regular NMF &  &            & 0.417 & 0.308 & 0.430 & 0.455 & 0.448 \\
 \cmidrule(lr){2-9}
 & \multirow{5}{*}{\(\alpha\)-NMF}
 & 
   & 0.01 & 0.020 & 0.012 & -0.001 & 0.002 & 0.010 \\
 &  &              & 0.25 & 0.464 & \textbf{0.475} & 0.449 & 0.429 & 0.437 \\
 &  &              & 0.50 & 0.433 & 0.391 & 0.449 & 0.448 & 0.393 \\
 &  &              & 0.75 & 0.391 & 0.406 & \textbf{0.490} & 0.449 & 0.441 \\
 &  &              & 0.99 & 0.374 & 0.370 & 0.389 & 0.432 & 0.410 \\
 \cmidrule(lr){2-9}
 & \multirow{15}{*}{Chem-NMF}
   & \multirow{5}{*}{0.01}
   & 0.01 & 0.110 & 0.105 & 0.102 & 0.109 & 0.100 \\
 & &                   & 0.25 & 0.180 & 0.175 & 0.170 & 0.182 & 0.165 \\
 & &                   & 0.50 & 0.210 & 0.205 & 0.200 & 0.215 & 0.190 \\
 & &                   & 0.75 & 0.245 & 0.230 & 0.238 & 0.250 & 0.220 \\
 & &                   & 0.99 & 0.280 & 0.270 & 0.265 & 0.278 & 0.245 \\
 \cmidrule(lr){3-9}
 &  & \multirow{5}{*}{0.10}
   & 0.01 & 0.220 & 0.215 & 0.210 & 0.218 & 0.208 \\
 &  &                   & 0.25 & 0.350 & 0.340 & 0.345 & 0.352 & 0.330 \\
 &  &                   & 0.50 & 0.420 & 0.405 & 0.415 & 0.425 & 0.390 \\
 &  &                   & 0.75 & 0.450 & 0.430 & 0.440 & 0.452 & 0.410 \\
 &  &                   & 0.99 & 0.465 & 0.445 & 0.455 & 0.460 & 0.425 \\
 \cmidrule(lr){3-9}
 &  & \multirow{5}{*}{0.50}
   & 0.01 & 0.032 & 0.010 & 0.002 & 0.011 & -0.003 \\
 &  &                   & 0.25 & 0.480 & 0.440 & 0.455 & 0.435 & 0.440 \\
 &  &                   & 0.50 & \textbf{0.510} & 0.465 & 0.480 & \textbf{0.475} & \textbf{0.450} \\
 &  &                   & 0.75 & 0.495 & 0.455 & 0.465 & 0.470 & 0.445 \\
 &  &                   & 0.99 & 0.470 & 0.435 & 0.450 & 0.460 & 0.430 \\
\bottomrule
\end{tabular}
\end{table}

\newpage
\section*{Supplementary Material}

Catalysts play a crucial role in increasing the rate of chemical reactions by providing an alternative pathway with lower activation energy. We present an illustrative example of a catalytic action through the hydrogenation of alkenes, specifically the conversion of ethene to ethane:

\begin{equation}
\mathrm{C_2H_{4(g)}+ H_{2 (g)} \xrightarrow{Ni _{(s)}} C_2H_{6 (g)}}
\tag {S1}
\end{equation}

Alkenes contain a carbon–carbon double bond, which is relatively weak and highly reactive toward addition reactions. In a hydrogenation reaction, hydrogen atoms add across the double bond, yielding a saturated alkane. Although this process is thermodynamically favourable, it does not occur without a catalyst due to the high activation energy barrier. The catalyst enables the reaction by lowering this barrier [1]. In heterogeneous catalysis, the catalyst exists in a different phase from the reactants, typically a solid metal surface with gaseous reactants. The reaction proceeds via the adsorption of hydrogen and alkene molecules on the catalyst surface, followed by bond dissociation and the subsequent addition of hydrogen to the double bond. Ultimately, the product molecules separate from the catalyst surface, a process known as desorption (Fig. \ref{fig:chem1}). Common industrial catalysts include nickel, palladium, and platinum. Because the catalyst is not consumed in the process, it can be reused multiple times. Hydrogenation is widely applied in the chemical and food industries. For instance, in converting unsaturated oils into semi-solid fats such as margarine, thereby improving product stability and shelf life [2].

\renewcommand{\thefigure}{S\arabic{figure}}
\setcounter{figure}{0}
\begin{figure}[H]
    \centering
    \includegraphics[width=\textwidth]{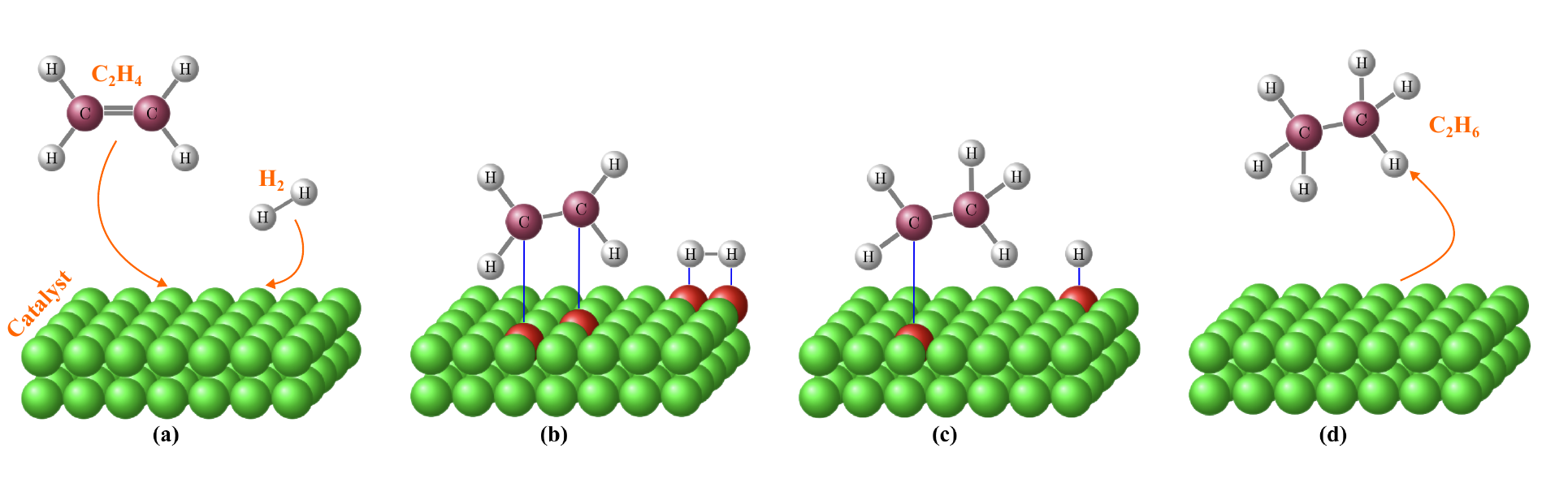}
    \captionsetup{font=footnotesize}
    \caption{%
Schematic of the catalytic hydrogenation of ethene to ethane over a heterogeneous nickel catalyst. 
\textbf{(a)} Adsorption of the reactant molecules onto the metal surface; 
\textbf{(b)} Dissociation of the {H--H} and {C=C} bonds; 
\textbf{(c)} Migration and addition of hydrogen atoms to the carbon atoms; \textbf{(d)} Desorption of the ethane product from the catalyst surface.}
\label{fig:chem1}
\end{figure}

\vspace{2mm}
\noindent\textbf{References}
\footnotesize
\begin{enumerate}
    \item ``Catalytic Hydrogenation,'' \emph{LibreTexts Chemistry}, 2025. [Online]. \\ \url{https://chem.libretexts.org/Bookshelves/Organic_Chemistry}
    \item ``Applications of Heterogeneous Catalysis in Industry,'' \emph{Cademix Institute of Technology}, 2025. [Online]. \\ \url{https://www.cademix.org/applications-of-heterogeneous-catalysis-in-industry/}
\end{enumerate}

\end{document}